\documentclass[conference]{IEEEtran}
\IEEEoverridecommandlockouts

\usepackage{cite}
\usepackage{amsmath,amssymb,amsfonts}
\usepackage{algorithmic}
\usepackage{graphicx}
\usepackage{textcomp}
\usepackage{xcolor}
\usepackage{multirow}
\usepackage{booktabs}
\usepackage{tabularx}
\usepackage{colortbl}
\usepackage[ruled,linesnumbered,vlined]{algorithm2e}
\usepackage{makecell}
\SetKwInput{KwIn}{Input}
\SetKwInput{KwOut}{Output}
\SetKwComment{tcp}{$\triangleright$~}{}
\SetCommentSty{textit}
\usepackage{caption}
\usepackage[breakable]{tcolorbox}
\definecolor{darkpurple}{HTML}{1F3864} 
\definecolor{leakred}{RGB}{153,0,0} 
\usepackage{fontawesome5}
\usepackage[
    colorlinks=true,
    linkcolor=darkpurple,
    citecolor=darkpurple,
    urlcolor=darkpurple
]{hyperref}
\usepackage[normalem]{ulem}
\definecolor{pos}{RGB}{0,128,0}
\definecolor{neg}{RGB}{178,34,34}
\definecolor{best}{HTML}{C9E0F5}
\definecolor{second}{HTML}{E8F4FC}
\definecolor{rowgray}{gray}{0.92}
\newcommand{\up}{\textcolor{red}{$\uparrow$}}
\usepackage{amsthm}
\newtheorem{theorem}{Theorem}
\newtheorem{remark}{Remark}

\providecommand{\Description}[1]{}

\def\BibTeX{{\rm B\kern-.05em{\sc i\kern-.025em b}\kern-.08em
    T\kern-.1667em\lower.7ex\hbox{E}\kern-.125emX}}

\begin{document}

\title{GRACE: LLM-Grounded Semantic Metric Spaces for Scalable Mixed-Data Clustering}

\author{
Zihua Yang$^{1}$, Zhencheng Xie$^{1}$, Junyang Chen$^{2}$, Liang Xie$^{3}$, Yiqun Zhang$^{1}$, Mengke Li$^{4}$, Yang Lu$^{5}$
\\
$^{1}$\textit{Guangdong University of Technology}, 
$^{2}$\textit{Tsinghua University}\\
$^{3}$\textit{Peking University},
$^{4}$\textit{Shenzhen University},
$^{5}$\textit{Xiamen University}\\
\{yangzihua1, xiezhencheng\}@mails.gdut.edu.cn, chenjuny25@mails.tsinghua.edu.cn, \\liangxie@pku.edu.cn, yqzhang@gdut.edu.cn, 
mengkeli@szu.edu.cn, luyang@xmu.edu.cn
}

\maketitle

\begin{abstract}
Clustering mixed tabular data requires a unified metric space to bridge the inherent heterogeneity between continuous numerical measurements and discrete categorical symbols. Traditionally, algorithms rely entirely on dataset-internal statistics to estimate categorical relationships, which confines the learned metric to empirical co-occurrences and ignores conceptually obvious yet statistically unobserved affinities. Although LLMs offer external world knowledge, applying their text-centric reasoning to highly abstract tabular concepts presents significant challenges. Bridging this modality gap to construct a semantically complete metric typically requires embedding LLMs into iterative metric learning loops to dynamically optimize cross-modality representations. This incurs intractable computational overhead, forcing a compromise between semantic enrichment and scalability. Therefore, we propose GRACE, an LLM-grounded framework for scalable mixed-data clustering. GRACE shifts semantic acquisition to the attribute-value level via a multi-perspective LLM querying strategy, mapping heterogeneous values into knowledge-informed descriptions. Crucially, this one-shot grounding extracts general-purpose semantic representations that embed heterogeneous attributes into a unified space, decoupling expensive LLM invocation from iterative optimization. Furthermore, GRACE cross-validates these external semantics against dataset-internal statistical evidence to ensure alignment with the dataset-specific cluster structure. Ultimately, GRACE matches the scalability of conventional statistics-driven baselines while achieving superior clustering accuracy and conceptual interpretability over 11 competing methods. The source code is available at \href{https://github.com/develop-yang/GRACE-GRACE-A}{\faGithub~GRACE}.
\end{abstract}


\section{Introduction}

Tabular data with heterogeneous numerical and categorical attributes is ubiquitous in healthcare, finance, and e-commerce~\cite{pu2025leveraging,shields2025designer}. In these domains, clustering supports label-free organization of large-scale records for exploratory analysis, cohort discovery, and data-driven decision-making. Clustering such mixed data requires measuring pairwise similarity across all attributes~\cite{spurio2025hierarchical,shi2025speheatal}. For numerical attributes, value differences provide a readily defined distance, but categorical attributes take values from finite unordered sets and carry no intrinsic distance~\cite{kim2025predict}. Whether \textit{nurse} is closer to \textit{doctor} or to \textit{engineer}, for example, cannot be judged from their symbolic forms alone. Categorical distance is therefore a key component of mixed-data similarity without a natural definition, and how well it is constructed governs clustering performance~\cite{juntong2025tabdiff,andreas2024text}. 

To estimate categorical distances, one line of research defines distance functions directly in the original feature space, using statistical signals of varying granularity, from attribute-wise entropy weighting and cross-attribute dependency modeling~\cite{zhu2018heterogeneous,zhu2020unsupervised} to higher-order value interaction graphs~\cite{zhang2022graph}. These methods recover relationships that symbolic forms alone cannot reveal, but only when such relationships are reflected by co-occurrence, dependency, or interaction patterns in the observed data. Representation learning, by contrast, maps heterogeneous attributes into a shared continuous space where standard distance metrics, such as Euclidean or cosine distance, can be applied~\cite{zhang2025learning,chen2024qgrl}, enabling unified treatment of numerical and categorical attributes. Such embeddings capture dependencies that feature-space metrics cannot express~\cite{jonathan2024interpretable}. However, their learning signal remains derived from dataset-internal statistics. As a result, the learned representations mainly reorganize empirical regularities already present in the observed data~\cite{guillaume2025mixed}. Taken together, these methods share a common limitation: categorical relationships are inferred only from evidence internal to the dataset. This reliance on internal statistics is restrictive because categorical values often denote real-world concepts with readily available external semantics, and ignoring such prior knowledge can leave obvious conceptual affinities invisible to the learned metric. This motivates the use of external semantic knowledge for mixed-data clustering.

To meet the demand for semantic knowledge beyond the dataset, the recent emergence of LLMs offers a natural source, as these models intrinsically encode world knowledge about real-world concepts and their relationships~\cite{zhang2025trending,du2025information}. Much like humans infer the closeness between occupations from domain knowledge rather than from their bare names, grounding attribute values with LLM knowledge lifts isolated symbols into a semantic metric space where conceptual affinities become geometrically structured. Sparse symbol encoding derives distances from observed data patterns, leaving the induced metric space semantically incomplete. As a result, conceptual proximity that is evident from world knowledge may not be faithfully translated into geometric proximity. Once each value is expanded into descriptive text grounded in LLM world knowledge, the external semantics reshape the metric space, making domain-level associations explicit and better aligned with the underlying clustering structure.

The remaining challenge is how to turn this external semantic prior into a clustering-ready metric space for mixed data. Existing LLM-based clustering mainly exploits LLMs in their native language space, where each instance is already a sentence or document and the model can directly reason about semantic similarity~\cite{zhang2023clusterllm,viswanathan2024large}. This text-centric paradigm does not transfer naturally to mixed tabular data. Since categorical values are isolated but semantically condensed concepts, whereas numerical values are computable measurements with their own metric structure, both are expected to be represented within a unified similarity space. In principle, the semantic compression of categorical values makes LLM guidance especially valuable for metric construction. However, repeatedly invoking an LLM during semantic grounding, neighborhood construction, or cluster refinement would introduce substantial latency and computational overhead, especially when pairwise or neighborhood structures scale with the data. Therefore, the central obstacle to applying LLM knowledge to mixed-data clustering is the efficiency bottleneck of obtaining semantic guidance at scale. This calls for a new paradigm that acquires LLM-grounded semantic representations without repeated invocation, while still aligning the resulting space with the dataset-specific clustering structure.

Guided by this observation, this paper proposes GRACE (\textbf{GR}ounding \textbf{A}ttributes for \textbf{C}lustering via \textbf{E}xternal semantics), a framework that makes LLM world knowledge usable for mixed-data clustering without embedding the LLM in iterative metric learning or cluster refinement. The key idea is to shift semantic acquisition from the instance or pair level to the attribute-value level. Rather than repeatedly querying the LLM over samples, neighborhoods, or intermediate cluster states, GRACE grounds each possible attribute value once: categorical symbols are expanded into knowledge-grounded descriptions, and numerical domains are partitioned into domain-informed intervals described in the same semantic form. This value-level design decouples LLM knowledge acquisition from clustering-time computation and turns heterogeneous values into comparable semantic units for a unified metric space.
External semantics alone, however, do not determine the cluster structure of a specific dataset. A semantic affinity may be conceptually plausible yet irrelevant to the empirical grouping induced by the observed records. GRACE therefore complements LLM-grounded semantics with dataset-internal statistical evidence. The semantic space provides external conceptual knowledge, and the statistical view calibrates it by identifying neighborhood relations supported by the original features. Relations confirmed by both views are selectively enhanced, yielding an affinity structure that preserves world-knowledge-aware similarity and remains aligned with dataset-specific clustering structure. Finally, eigengap-guided graph sparsification converts the refined affinity matrix into a spectrally well-conditioned graph for clustering.
The main contributions of this paper are summarized below:
\begin{itemize}
    \item \textbf{Bridging LLMs and mixed data clustering.} This work is the first to bring external semantic knowledge into effective complementarity with dataset-internal statistical evidence for mixed-data clustering, constructing a unified world-knowledge-aware metric space that captures both conceptual affinities and distributional structures.
    \item \textbf{Aligning semantic knowledge with cluster structure.} GRACE resolves the misalignment between semantic proximity and cluster membership through dual-view neighborhood consistency, providing a principled mechanism for coupling external semantics with dataset-specific statistical evidence.
     \item \textbf{Decoupling LLM grounding from metric learning.} GRACE reformulates LLM-assisted mixed-data metric learning as a one-shot value-level grounding process followed by LLM-free metric adaptation. Since the LLM is invoked only over possible attribute values rather than instances, the semantic acquisition cost scales linearly with value cardinality, providing a low-token-overhead and scalable route to world-knowledge-aware clustering.
\end{itemize}

\section{Related Work}
\subsection{Distance Measures for Mixed Data}
Categorical values carry no inherent distance, and the earliest measures fill that void by counting matches, treating every mismatch as equally far.
Such matching yields a workable metric, yet it flattens the unequal relatedness that real concepts exhibit.
To recover that relatedness, a long line of work mines it from the data, advancing through co-occurrence context~\cite{ienco2012context}, entropy weighting, cross-attribute couplings~\cite{zhu2018heterogeneous,zhu2020unsupervised}, learnable intra-attribute weights, higher-order value graphs~\cite{zhang2022graph}, and dense embeddings of the same couplings~\cite{jian2017embedding,jian2018cure}.
Each step reads the statistics more finely than the last, yet every one inherits the same blind spot, namely, any affinity that leaves no trace in the observed sample cannot be recovered.

Mixed data, where numerical and categorical attributes coexist, raise a second problem, that of placing continuous measurements and discrete symbols on one scale.
Early schemes settle it by hand, weighting a Euclidean term against a matching one (k-prototypes~\cite{huang1998extensions}, Gower~\cite{gower1971general}), and later ones learn a shared space that absorbs both types (unified distance learning~\cite{zhang2025learning}, quaternion graphs~\cite{chen2024qgrl}).
The joint metric grows more principled with each design, yet the cure for the scale problem leaves the earlier blind spot untouched, since meaning is still drawn entirely from within the dataset.
Two values that never co-occur are therefore judged far apart even when domain knowledge would call them close, a verdict no statistical refinement can overturn.

\begin{figure*}[t]
    \centering
    \includegraphics[width=\textwidth]{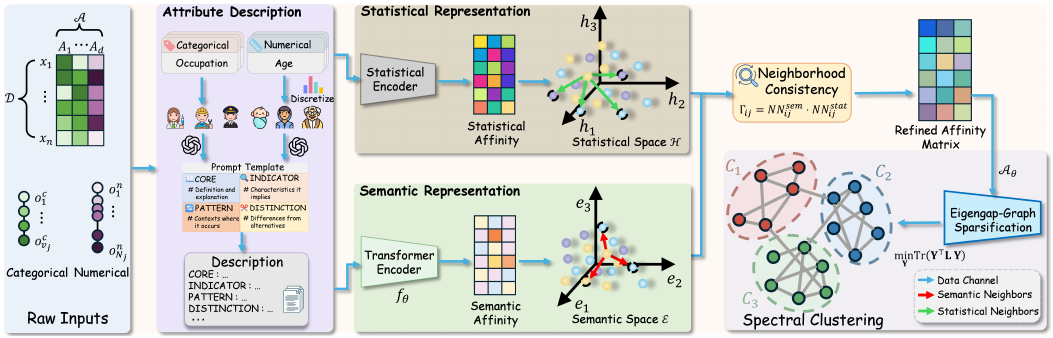}
    \Description{Overview of the GRACE framework.}
    \caption{Overview of the GRACE framework. Given raw tabular data with categorical and numerical attributes, GRACE generates natural-language descriptions for each attribute value through LLM-based prompting. The generated descriptions are fed into a transformer encoder $f_\theta$ to obtain semantic representations $\mathbf{e}_i \in \mathcal{E}$. To improve the reliability of semantic affinity, a statistical encoder produces representations $\mathbf{h}_i \in \mathcal{H}$ from the raw attribute values for validating semantic neighbors. To jointly exploit both spaces, a neighborhood consistency mechanism computes $\Gamma_{ij} = \mathrm{NN}_{ij}^{\mathrm{sem}} \cdot \mathrm{NN}_{ij}^{\mathrm{stat}}$, refining the semantic affinity matrix. Based on the refined affinity, eigengap-graph sparsification determines the optimal sparsity for spectral clustering.}

    \label{fig:framework}
\end{figure*}

\subsection{LLM-Enhanced Clustering}

LLMs encode world knowledge that statistical distances alone cannot recover, and recent clustering research draws on it. These efforts, however, remain confined to inherently textual input. ClusterLLM~\cite{zhang2023clusterllm} elicits triplet judgments of relative proximity, few-shot prompting~\cite{viswanathan2024large} converts limited supervision into pseudo-labels or pairwise constraints, loop-based refinement~\cite{an2024generalized} corrects assignments from model feedback, and contrastive formulations such as SCCL~\cite{zhang2021supporting} and TAC~\cite{li2024tac} incorporate the language signal into the clustering objective. Tabular data, by contrast, offers no such linguistic surface, as its categorical entries are discrete symbols and its numerical entries scalar quantities. A knowledge-grounded distance over such heterogeneous attributes therefore remains an open question.

Applying language models to tabular data is itself well studied, though existing work targets prediction or generation instead of a clustering metric. One line serializes each record into a sentence, as in TabLLM~\cite{hegselmann2023tabllm} for classification and GReaT~\cite{borisov2023great} for generation, and another recasts column semantics as features for a downstream learner, as in CAAFE~\cite{hollmann2023large} and FeatLLM~\cite{han2024large}. Both lines serve supervised objectives and therefore provide no unsupervised distance over mixed attributes. The effort closest to such a distance is BREVE~\cite{yang2026bridging}, which grounds categorical values in model knowledge to derive a similarity measure. Its fusion is nonetheless static, embedding raw value identities into a shared space under a fixed, non-trainable weight, and cannot separate genuine from spurious semantic affinities. It is moreover restricted to categorical attributes, excluding numerical ones.

\begin{table}[t]
\centering
\small
\caption{Frequently used symbols and notations.}
\label{tab:notation}
\begin{tabularx}{\columnwidth}{c|X}
\toprule
\textbf{Symbol} & \textbf{Description} \\
\midrule
\rowcolor{rowgray}
$\mathcal{D}$, $n$ & Dataset and the number of 
samples \\
$x_i$, $\mathcal{A}$, $m$ & Sample, attribute set, 
and attribute count \\
\rowcolor{rowgray}
$\mathcal{A}^c$, $\mathcal{A}^n$ & Categorical and 
numerical attribute subsets \\
$K$, $\mathcal{C}$ & Number of clusters and 
partition \\
\rowcolor{rowgray}
$t(j, o)$ & LLM-generated description for value $o$ 
on attribute $A_j$ \\
$f_\theta$, $d_e$ & Sentence encoder and embedding 
dimensionality \\
\rowcolor{rowgray}
$\mathbf{e}(j, o)$ & Attribute-level embedding for 
value $o$ on $A_j$ \\
$\mathbf{z}_i$, $\mathbf{h}_i$ & Semantic and 
statistical representations of sample $x_i$ \\
\rowcolor{rowgray}
$\mathbf{S}^{\mathrm{sem}}$, 
$\mathbf{S}^{\mathrm{stat}}$ & Semantic and 
statistical similarity matrices \\
$\lambda$ & Natural neighbor search radius \\
\rowcolor{rowgray}
$\mathbf{NN}^{\mathrm{sem}}$, 
$\mathbf{NN}^{\mathrm{stat}}$ & Natural neighbor 
graphs for two views \\
$\boldsymbol{\Gamma}$, $\tau_c$ & Consistency 
matrix and adaptive threshold \\
\rowcolor{rowgray}
$\rho_{il}$, $\alpha_{il}$ & Refined affinity and 
enhancement coefficient \\
$\theta^*$, $\mathbf{A}_{\theta^*}$ & Optimal 
sparsification threshold and adjacency matrix \\
\rowcolor{rowgray}
$\mathrm{REQ}$ & Relative Eigengap Quality \\
\bottomrule
\end{tabularx}
\end{table}

\section{Proposed Method}

GRACE follows a four-stage pipeline consisting of problem formulation, knowledge-grounded attribute description, dual-view neighborhood consistency, and model training with spectral clustering, as illustrated in Fig.~\ref{fig:framework}. For clarity, the frequently used symbols are summarized in Table~\ref{tab:notation}. The complexity of both GRACE and its scalable variant GRACE-A is analyzed in the complexity analysis subsection.

\subsection{Problem Formulation}
\label{sec:problem}

Given a dataset $\mathcal{D} = \{x_1, x_2, \ldots, x_n\}$ with $n$ samples, each sample $x_i = [x_i^1, x_i^2, \ldots, x_i^m]^\top$ is described by $m$ attributes $\mathcal{A} = \{A_1, A_2, \ldots, A_m\}$, where $x_i^j$ denotes the value of sample $x_i$ on attribute $A_j$. These attributes are partitioned into a categorical subset $\mathcal{A}^c$ and a numerical subset $\mathcal{A}^n$, satisfying $\mathcal{A} = \mathcal{A}^c \cup \mathcal{A}^n$ with $\mathcal{A}^c \cap \mathcal{A}^n = \emptyset$. Specifically, a categorical attribute $A_j \in \mathcal{A}^c$ takes values from a finite unordered set $\mathcal{O}^c_j = \{o^c_1, o^c_2, \ldots, o^c_{V_j}\}$ of cardinality $V_j$, where values lack intrinsic ordering, and a numerical attribute $A_j \in \mathcal{A}^n$ takes values from a continuous interval $\mathcal{O}^n_j = [\ell_j, u_j] \subseteq \mathbb{R}$, with observed values $\{o^n_1, o^n_2, \ldots, o^n_{N_j}\}$.

Clustering such heterogeneous data requires measuring pairwise similarity across both attribute types. The goal is to partition the $n$ samples into $K$ disjoint groups $\mathcal{C} = \{C_1, C_2, \ldots, C_K\}$ such that similar samples fall into the same group. To achieve the partition, GRACE adopts spectral clustering, which operates on a similarity graph $\mathcal{G} = (\mathcal{D}, \mathbf{S})$ in which the edge weight $S_{ij}$ quantifies the affinity between $x_i$ and $x_j$. Then, an indicator matrix $\mathbf{Y} \in \mathbb{R}^{n \times K}$ encodes the group assignment and is obtained by minimizing the normalized graph cut:
\begin{equation}
\min_{\mathbf{Y}} \operatorname{Tr}(\mathbf{Y}^\top \mathbf{L} \mathbf{Y}) \quad \text{s.t.} \quad \mathbf{Y}^\top \mathbf{D} \mathbf{Y} = \mathbf{I},
\label{eq:ncut}
\end{equation}
where $\mathbf{L} = \mathbf{I} - \mathbf{D}^{-1/2} \mathbf{S}\mathbf{D}^{-1/2}$ is the normalized Laplacian and $\mathbf{D}$ denotes the diagonal degree matrix with entries $\mathbf{D}_{ii} = \sum_j S_{ij}$. The standard spectral relaxation solves \hyperref[eq:ncut]{Eq.~\eqref{eq:ncut}} by extracting the $K$ smallest eigenvectors of $\mathbf{L}$ and applying $k$-means to the rows of the resulting $\mathbf{Y}$. Since $\mathbf{L}$ is constructed entirely from $\mathbf{S}$, the quality of the final partition depends on how $\mathbf{S}$ captures the true grouping structure. For categorical attributes whose values lack inherent metric structure, defining meaningful entries in $\mathbf{S}$ is the core challenge. GRACE addresses this by grounding each attribute value in LLM world knowledge to derive semantically grounded $\mathbf{S}$.

\begin{figure}[t]
    \centering
    \includegraphics[width=\columnwidth]{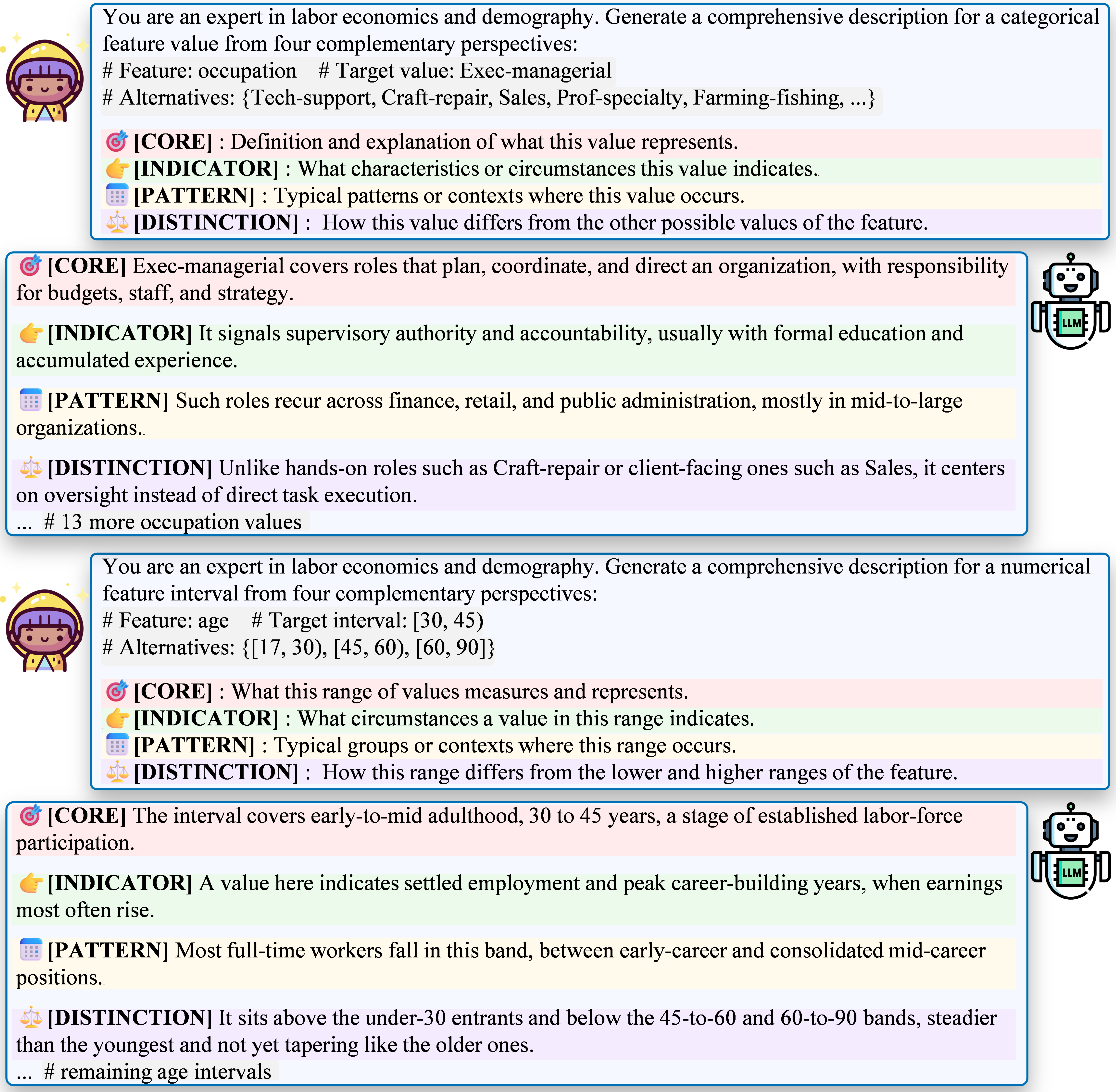}
    \Description{prompt}
    \caption{The $4P$ prompt applied to two representative attributes from the Adult dataset: the categorical value \textit{Exec-managerial} of \textit{occupation} and the numerical interval $[30,45]$ of \textit{age}. For both attribute types, the prompt elicits four complementary semantic perspectives, producing knowledge-grounded descriptions in a unified textual format.}
    \label{fig:prompt}
\end{figure}

\subsection{LLM-Grounded Attribute Description}
\label{sec:semantic}

To unify heterogeneous attribute types into a common representation, GRACE converts every attribute value into a natural-language description generated by an LLM. The description quality depends on prompt design. A naive prompt such as ``\textit{Describe the value has 4 legs}'' produces nearly identical embeddings for semantically distinct values, because it fails to elicit functional associations (e.g., four-legged animals are typically mammals, six-legged ones are insects) or contrasts with alternatives.

Motivated by this, GRACE adopts a four-perspective knowledge elicitation template, denoted $4P$ (\hyperref[fig:prompt]{Fig.~\ref{fig:prompt}}). Within a single query, the template instructs the LLM to produce a structured description for each categorical value $o \in \mathcal{O}^c_j$:
\begin{equation}
t(j, o) = \mathrm{LLM}_{4P}(A_j,\, o),
\label{eq:llm_desc}
\end{equation}
where $\mathrm{LLM}_{4P}$ denotes the language model queried under the four-perspective template, $A_j$ is the attribute name, and $o$ is the value being described. The generated description $t(j, o)$ consists of four complementary components, denoted $t_{\textsc{core}}$, $t_{\textsc{ind}}$, $t_{\textsc{pat}}$, and $t_{\textsc{dis}}$, corresponding to the definitional, functional, contextual, and discriminative perspectives respectively. The first two capture the internal semantics of a value, i.e., what it means and what it implies, and the latter two encode external relationships, i.e., where it co-occurs and how it contrasts with alternatives. The four components are concatenated into a single multi-faceted textual passage that provides rich semantic grounding for downstream encoding.

The $4P$ template additionally includes a domain-adaptive role assignment and quality constraints that require concrete, jargon-free language with domain-specific examples.

The four-perspective template applies directly to categorical attributes, where the set of possible values is finite, and each value receives its own description. For numerical attributes, the set of observed values is typically continuous, making per-value descriptions impractical. GRACE delegates discretization to the LLM, providing the attribute name, observed range, and domain context. The LLM determines both the number of intervals $Q_j$ and their boundaries from domain knowledge, e.g., clinical reference ranges for medical indicators. Formally, the discretization defines a mapping:
\begin{equation}
\psi_j: \mathcal{O}^n_j \;\to\; \mathcal{I}_j,\quad \mathcal{I}_j = \{I_1, I_2, \ldots, I_{Q_j}\},
\label{eq:interval}
\end{equation}
where $\psi_j$ assigns each observed value to one of $Q_j$ intervals whose union covers the entire observed range $[\ell_j, u_j]$. As a concrete example, bilirubin is partitioned into three intervals, i.e., normal (below 1.2\,mg/dL), elevated (1.2 to 3.0\,mg/dL), and severe (above 3.0\,mg/dL), corresponding to the standard diagnostic thresholds in hepatology. Each interval $I_k$ is then described through the same four-perspective query as in \hyperref[eq:llm_desc]{Eq.~\eqref{eq:llm_desc}}, namely $t(j, I_k) = \mathrm{LLM}_{4P}(A_j, I_k)$, giving categorical values and numerical intervals one common description format. All downstream modules therefore operate on a single textual input space, regardless of the original attribute type.

Given these textual descriptions, GRACE encodes each $t(j, o)$ into a $d_e$-dimensional dense embedding $\mathbf{e}(j, o) \in \mathbb{R}^{d_e}$ via a transformer encoder $f_\theta$, which can be expressed as:
\begin{equation}
\mathbf{e}(j, o) = f_\theta\bigl(t(j, o)\bigr),
\label{eq:encode}
\end{equation}
To obtain a sample-level representation $\mathbf{z}_i \in \mathbb{R}^{m \cdot d_e}$ from the $m$ attribute-level embeddings, GRACE concatenates them into a single vector:
\begin{equation}
\mathbf{z}_i = \mathbf{e}(1, o_i^1) \oplus \mathbf{e}(2, o_i^2) \oplus \cdots \oplus \mathbf{e}(m, o_i^m),
\label{eq:concat}
\end{equation}
with $o_i^j$ denoting the (possibly discretized) value of sample $x_i$ on attribute $A_j$. The pairwise \emph{semantic similarity} between two samples is then computed as:
\begin{equation}
S_{il}^{\mathrm{sem}} = \frac{\mathbf{z}_i^\top \mathbf{z}_l}{\|\mathbf{z}_i\|\,\|\mathbf{z}_l\|}.
\label{eq:sem_sim}
\end{equation}

\begin{remark}
\label{rem:align}
\textbf{Robustness to Keyword Alignment.}
GRACE assigns embeddings at the attribute-value level (\hyperref[eq:encode]{Eq.~\eqref{eq:encode}}). Thus, samples sharing the same value $o$ on attribute $A_j$ receive an identical vector $\mathbf{e}(j, o)$ without requiring keyword alignment across different descriptions. Because concatenation allocates disjoint dimensions to each attribute, the inner product $\mathbf{z}_i^\top \mathbf{z}_l$ decomposes into per-attribute terms. For every attribute on which $x_i$ and $x_l$ agree, the corresponding term equals $\|\mathbf{e}(j, o)\|^2$, a constant determined solely by value identity. Misaligned keywords can only affect the terms for attributes on which the two samples differ. Within-cluster sample pairs typically share more attribute values than between-cluster pairs, making their similarity estimates rely more heavily on value-identity anchors than on potentially misaligned cross-value terms. Keyword misalignment therefore has a bounded influence on the semantic affinity, which explains the robustness of GRACE at the clustering level.
\end{remark}

\subsection{Dual-View Neighborhood Consistency}
\label{sec:dsnc}


The semantic similarity matrix $\mathbf{S}^{\mathrm{sem}}$ quantifies pairwise conceptual proximity among samples, but in an unsupervised setting, no external labels are available to assess whether specific estimates are accurate. A principled remedy is to cross-check against the raw tabular features, which encode sample relationships through observed co-occurrence patterns.

GRACE implements this cross-check through a statistical view that operates on the original feature space. Because the statistical view serves a purely confirmatory role, an expressive encoding such as entity embeddings~\cite{guo2016entity} or learned tabular representations~\cite{gorishniy2021revisiting} would introduce trainable parameters that require labeled supervision or auxiliary objectives to tune, conflicting with the unsupervised setting. GRACE therefore adopts a parameter-free encoding.
For each categorical attribute $A_j \in \mathcal{A}^c$, every value $o \in \mathcal{O}^c_j$ is mapped to a standard basis vector $\mathbf{b}_o \in \mathbb{R}^{V_j}$, and for each numerical attribute $A_j \in \mathcal{A}^n$, observed values are normalized to $[0,1]$ via min-max scaling.
Concatenating these attribute-level representations yields a feature vector $\mathbf{h}_i \in \mathbb{R}^{d_h}$ with dimensionality $d_h = \sum_{A_j \in \mathcal{A}^c} V_j + |\mathcal{A}^n|$.
The resulting pairwise statistical similarity between samples $x_i$ and $x_l$, denoted by $S^{\mathrm{stat}}_{il}$, is then defined as:
\begin{equation}
S_{il}^{\mathrm{stat}} = \frac{\mathbf{h}_i^\top \mathbf{h}_l}{\|\mathbf{h}_i\|\,\|\mathbf{h}_l\|}.
\label{eq:stat_sim}
\end{equation}


Given $\mathbf{S}^{\mathrm{sem}}$ and $\mathbf{S}^{\mathrm{stat}}$, the next step is to identify neighbors under each view. Comparing raw values across views is problematic because the two matrices operate on different scales. Comparing \emph{local neighborhood structures} is more principled, since clustering depends on relative proximity, not absolute magnitude.
To this end, GRACE extracts neighborhood structures using Natural Neighbor (NaN) graphs~\cite{zhu2016natural}, which impose a mutual recognition constraint.
Specifically, a sample $x_l$ qualifies as a natural neighbor of $x_i$ only if $x_i$ also appears among the nearest neighbors of $x_l$:
\begin{equation}
\mathrm{NN}_{il}(\mathbf{S}) = \mathbb{I}\bigl[l \in \mathcal{N}_\lambda(i;\mathbf{S})\bigr] \cdot \mathbb{I}\bigl[i \in \mathcal{N}_\lambda(l;\mathbf{S})\bigr],
\label{eq:nn}
\end{equation}
in which $\mathbf{S}$ is a generic similarity matrix standing for either view, and $\mathcal{N}_\lambda(i;\mathbf{S})$ denotes the $\lambda$-nearest neighbors of $x_i$ under $\mathbf{S}$.
The mutual constraint is well-suited to the verification task, because one-directional affinities, i.e., cases where $x_i$ considers $x_l$ a neighbor but not vice versa, often arise in regions of varying density and tend to be unreliable indicators of cluster membership.
By requiring reciprocation, the NaN graph retains only the most structurally stable relationships in each view, and the neighborhood size $\lambda$ is determined automatically by iterating from $\lambda = 1$ until every sample possesses at least one natural neighbor.


Instantiating \hyperref[eq:nn]{Eq.~\eqref{eq:nn}} with the semantic matrix $\mathbf{S}^{\mathrm{sem}} = [S^{\mathrm{sem}}_{il}]$ from \hyperref[eq:sem_sim]{Eq.~\eqref{eq:sem_sim}} and the statistical matrix $\mathbf{S}^{\mathrm{stat}} = [S^{\mathrm{stat}}_{il}]$ from \hyperref[eq:stat_sim]{Eq.~\eqref{eq:stat_sim}} produces the two view-specific neighborhood matrices:
\begin{equation}
\mathrm{NN}^{\mathrm{sem}}_{il} = \mathrm{NN}_{il}(\mathbf{S}^{\mathrm{sem}}), \qquad
\mathrm{NN}^{\mathrm{stat}}_{il} = \mathrm{NN}_{il}(\mathbf{S}^{\mathrm{stat}}).
\label{eq:nn_views}
\end{equation}
Their agreement defines the \emph{neighborhood consistency score} of a pair $(x_i, x_l)$:
\begin{equation}
\Gamma_{il} = \mathrm{NN}_{il}^{\mathrm{sem}} \cdot \mathrm{NN}_{il}^{\mathrm{stat}}.
\label{eq:consistency}
\end{equation}
A value of $\Gamma_{il} = 1$ thus marks a pair that the semantic and statistical views both recognize as neighbors.
Because the two views are constructed from independent information sources, a pair confirmed by both carries stronger evidence, consistent with the co-training principle~\cite{blum1998combining} that agreement between conditionally independent views yields reduced false-positive rates.


The reduced false-positive rate noted above is what makes a confirmed pair worth rewarding.
A proposition in the \href{https://github.com/develop-yang/GRACE-GRACE-A}{\underline{supplementary material}} formalizes the argument, proving that a dual-view-confirmed pair ($\Gamma_{il}=1$) carries a strictly higher posterior probability of being a true within-cluster pair than a pair supported by the semantic view alone.
GRACE therefore rewards dual-view confirmation by scaling the semantic affinity:
\begin{equation}
\rho_{il} = S_{il}^{\mathrm{sem}} \cdot \alpha_{il},
\label{eq:refine}
\end{equation}
whose entries form the refined affinity matrix $\mathbf{W}$ with $W_{il} = \rho_{il}$, through the coefficient:
\begin{equation}
\alpha_{il} = 1 + \Gamma_{il}, \qquad \Gamma_{il} \in \{0,1\}.
\label{eq:alpha}
\end{equation}
The multiplicative and proportional form is deliberate. An unconfirmed pair ($\Gamma_{il}=0$) keeps its original semantic value, which leaves the knowledge-grounded signal undistorted, whereas a confirmed pair is boosted in proportion to its existing semantic strength, letting statistical agreement reinforce the semantic evidence instead of overriding it.

\begin{remark}
\label{rem:asymmetric}
\textbf{Asymmetric View Design.}
GRACE assigns asymmetric roles to the two views. The semantic view supplies the primary affinity structure enriched by external knowledge, and the statistical view provides dataset-specific calibration from the original features. Instead of averaging the two signals, GRACE uses statistical agreement as selective evidence to enhance semantic affinity, preserving the external-knowledge signal and aligning it with the observed clustering structure.
\end{remark}

\subsection{Model Training and Spectral Clustering}
\label{sec:training}


The semantic affinity $S^{\mathrm{sem}}_{il}$ in \hyperref[eq:sem_sim]{Eq.~\eqref{eq:sem_sim}} is the cosine similarity between the sample vectors $\mathbf{z}_i$ and $\mathbf{z}_l$, and \hyperref[eq:concat]{Eq.~\eqref{eq:concat}} builds each $\mathbf{z}_i$ from the per-attribute embeddings $\mathbf{e}(j,o) = f_\theta(t(j,o))$ of \hyperref[eq:encode]{Eq.~\eqref{eq:encode}}. The encoder $f_\theta$ therefore governs the entire metric, and adapting it to the grounded descriptions is what reshapes $S^{\mathrm{sem}}$. Such adaptation calls for a contrastive objective, one that keeps descriptions of the same value close and drives those of different values apart under the cosine geometry of \hyperref[eq:sem_sim]{Eq.~\eqref{eq:sem_sim}}.

Concretely, GRACE fine-tunes $f_\theta$ with an InfoNCE loss~\cite{oord2018representation}. The objective is well suited to attribute descriptions since it preserves the semantic identity of each grounded value and separates alternative values within the same attribute context. For an anchor description $t(j,o)$, the positive example $t^+$ is obtained by encoding the same text with independently sampled dropout masks, following~\cite{gao2021simcse}, and the remaining descriptions in the mini-batch $\mathcal{B}$ serve as negatives.
The per-anchor loss is:
\begin{equation}
\mathcal{L}(t) = -\log \frac{\exp\bigl(\mathrm{sim}(f_\theta(t),\, f_\theta(t^+)) / \tau_e\bigr)}{\sum_{t' \in \mathcal{B}} \exp\bigl(\mathrm{sim}(f_\theta(t),\, f_\theta(t')) / \tau_e\bigr)},
\label{eq:infonce}
\end{equation}
in which $\mathrm{sim}(\cdot, \cdot)$ denotes cosine similarity and $\tau_e$ is a temperature hyperparameter.
Descriptions from different values of the same attribute naturally act as hard negatives, because they share domain context yet refer to distinct categorical states.
The full encoder objective averages over all anchors:
\begin{equation}
\mathcal{L}_{\mathrm{enc}} = \frac{1}{|\mathcal{B}|} \sum_{t \in \mathcal{B}} \mathcal{L}(t).
\label{eq:enc_loss}
\end{equation}
The encoder is optimized with AdamW~\cite{loshchilov2019decoupled} for $E$ epochs using a learning rate of $\eta$ and a batch size $|\mathcal{B}|$.

After training, the refined affinity matrix $\mathbf{W}$ from \hyperref[eq:refine]{Eq.~\eqref{eq:refine}} retains pairwise affinities for all sample pairs, including weak relations with limited contribution to the underlying cluster structure. Applying spectral clustering directly to such a fully connected affinity graph may propagate low-confidence edges into the spectral embedding, weakening cluster separation.
GRACE therefore converts $\mathbf{W}$ into a sparse binary graph by thresholding at a value $\theta$, producing the adjacency matrix $(\mathbf{A}_\theta)_{il} = \mathbb{I}[\rho_{il} > \theta]$.
The threshold $\theta$ controls the trade-off between edge preservation and graph sparsity. A small $\theta$ tends to retain unreliable weak affinities, whereas a large $\theta$ may remove structurally necessary edges and fragment the graph into isolated components.
To avoid threshold-induced fragmentation, GRACE first identifies the candidate set $\Theta^*$ of thresholds that achieve minimal graph fragmentation. Following the spectral-theoretic analysis of~\cite{fan2022simple}, GRACE then selects from $\Theta^*$ the value $\theta^*$ that maximizes the Relative Eigengap Quality (REQ) of the normalized Laplacian $\mathbf{L}_\theta$ of $\mathbf{A}_\theta$:
\begin{equation}
\theta^* = \arg\max_{\theta \in \Theta^*}\; \frac{\lambda_{K+1}(\mathbf{L}_\theta) - \bar{\lambda}_K(\mathbf{L}_\theta)}{\bar{\lambda}_K(\mathbf{L}_\theta) + \epsilon},
\label{eq:req}
\end{equation}
where $\lambda_{K+1}$ is the $(K\!+\!1)$-th smallest eigenvalue of $\mathbf{L}_\theta$, $\bar{\lambda}_K$ denotes the mean of the first $K$ eigenvalues, and $\epsilon = 10^{-6}$.
A larger REQ indicates a wider spectral gap, corresponding to better-separated clusters. The $K$ smallest eigenvectors of $\mathbf{L}_{\theta^*}$ are then extracted and row-normalized, and $k$-means is applied to the rows to produce $\mathcal{C} = \{C_1, \ldots, C_K\}$, following the normalized spectral clustering algorithm~\cite{ng2001spectral}.

\begin{algorithm}[t]
\small
\caption{\small{GRACE: GRounding Attributes for 
Clustering via External Semantics}}
\label{alg:grace}
\KwIn{Dataset $\mathcal{D}$ with attributes 
$\mathcal{A}$, cluster count $K$, threshold 
candidates $\Theta$}
\KwOut{Cluster assignments $\mathcal{C} = \{C_1, 
\ldots, C_K\}$}
\tcp{\colorbox{rowgray}{\textbf{Knowledge-Grounded 
Attribute Description}}}
\ForEach{$A_j \in \mathcal{A}$}{
    \eIf{$A_j \in \mathcal{A}^c$}{
        Generate description $t(j, o)$ for each 
        $o \in \mathcal{O}^c_j$ via LLM\;
    }{
        Discretize into $\mathcal{I}_j$ and 
        generate descriptions via LLM\;
    }
}
Fine-tune encoder $f_\theta$ with InfoNCE loss 
\tcp*{\hyperref[eq:enc_loss]{Eq.~\eqref
{eq:enc_loss}}}
$\mathbf{z}_i \gets \bigoplus_{j=1}^{m} 
f_\theta(t(j, x_i^j))$ for each $x_i$ 
\tcp*{\hyperref[eq:concat]{Eq.~\eqref{eq:concat}}}
Compute semantic similarity 
$\mathbf{S}^{\mathrm{sem}}$ 
\tcp*{\hyperref[eq:sem_sim]{Eq.~\eqref
{eq:sem_sim}}}
\tcp{\colorbox{rowgray}{\textbf{Dual-View 
Neighborhood Consistency}}}
Encode statistical features $\{\mathbf{h}_i\}$ and 
compute $\mathbf{S}^{\mathrm{stat}}$ 
\tcp*{\hyperref[eq:stat_sim]{Eq.~\eqref
{eq:stat_sim}}}
Build natural neighbor graphs 
$\mathbf{NN}^{\mathrm{sem}}, 
\mathbf{NN}^{\mathrm{stat}}$ 
\tcp*{\hyperref[eq:nn]{Eq.~\eqref{eq:nn}}}
$\boldsymbol{\Gamma} \gets 
\mathbf{NN}^{\mathrm{sem}} \odot 
\mathbf{NN}^{\mathrm{stat}}$ 
\tcp*{\hyperref[eq:consistency]{Eq.~\eqref
{eq:consistency}}}
Refine affinity $\rho_{il}$ via hierarchical 
enhancement 
\tcp*{\hyperref[eq:refine]{Eq.~\eqref{eq:refine}}}
\tcp{\colorbox{rowgray}{\textbf{Graph 
Sparsification and Spectral Clustering}}}
\ForEach{$\theta \in \Theta$}{
    Construct $\mathbf{A}_\theta$ and evaluate 
    $\mathrm{REQ}(\mathbf{L}_\theta)$ 
    \tcp*{\hyperref[eq:req]{Eq.~\eqref{eq:req}}}
}
Select $\theta^*$ maximizing REQ among thresholds 
with minimal fragmentation\;
\Return{Normalized spectral clustering on 
$\mathbf{A}_{\theta^*}$}
\end{algorithm}

\subsection{Complexity Analysis}
\label{sec:complexity}

The complete GRACE procedure is summarized in \hyperref[alg:grace]{Algorithm~\ref{alg:grace}}. Because the exact formulation forms a dense affinity over every sample pair and then performs spectral clustering on it, the eigendecomposition of the $n \times n$ Laplacian drives the cost to $O(n^3)$, which becomes prohibitive at the scale of large tables. To retain the same affinity design at a far lower cost, we further develop an approximation variant, denoted GRACE-A, where ``A'' stands for Approximation. GRACE-A replaces the dense pairwise graph with a sample-to-anchor bipartite graph over $p \ll n$ anchors.

GRACE-A rests on an anchor-based approximation of the affinity structure. A set of $p$ anchors is first selected by farthest-point sampling, which spreads them across the embedding space and keeps them representative of the latent clusters. Each sample is then linked only to these anchors through a sample-to-anchor affinity matrix $\mathbf{Z} \in \mathbb{R}^{n \times p}$, replacing the full $n \times n$ graph. A truncated factorization of $\mathbf{Z}$ recovers the leading spectral components of the implied global graph in the Nystr\"{o}m sense, after which $k$-means on the recovered embedding yields the partition. The construction thereby involves each sample only a constant number of times once $p$ is fixed, which lowers the order from cubic to linear in $n$.

The following theorems establish the time and space complexity of GRACE and GRACE-A. Throughout, $n$, $m$, $K$, and $p$ count the samples, attributes, clusters, and anchors, $d_e$ denotes the embedding dimension, and $V$ the number of distinct attribute values, following \hyperref[tab:notation]{Table~\ref{tab:notation}}.

\begin{theorem}
\label{thm:time_grace}
Standard GRACE runs in $O(n^3 + n^2 m d_e)$ time.
\end{theorem}
\begin{proof}
The procedure divides into value grounding, affinity construction, and spectral clustering. Grounding the $V$ distinct attribute values, namely description generation, encoder fine-tuning, and value-level embedding, scales with $V$ instead of $n$ and costs $O(V d_e)$. Assembling the $n$ sample vectors and evaluating the dense semantic affinity over all $n^2$ pairs then takes $O(n^2 m d_e)$, since every entry is an inner product of two $(m d_e)$-dimensional vectors. The dual-view refinement manipulates $n \times n$ matrices at $O(n^2)$, a term the spectral stage subsumes. Eigendecomposing the normalized Laplacian $\mathbf{L}_\theta$ across a fixed budget of candidate thresholds requires $O(n^3)$, whereas the ensuing $k$-means adds only $O(t n K^2)$. Given $n \gg K$, the eigendecomposition and the affinity construction govern the total, which gives $O(n^3 + n^2 m d_e)$.
\end{proof}

\begin{theorem}
\label{thm:space_grace}
Standard GRACE requires $O(n^2)$ space.
\end{theorem}
\begin{proof}
Storage is dictated by the largest object the pipeline keeps in memory. The dense affinity matrix occupies $O(n^2)$ and dominates every other term, since the input table $\mathcal{D} \in \mathbb{R}^{n \times m}$, the value embeddings, and the spectral embedding $\mathbf{Y} \in \mathbb{R}^{n \times K}$ need only $O(nm)$, $O(n m d_e)$, and $O(nK)$ respectively, each below $O(n^2)$ whenever $n$ exceeds $m d_e$. The space complexity is therefore $O(n^2)$.
\end{proof}

GRACE-A lowers both bounds to linear order, as the next two theorems show.

\begin{theorem}
\label{thm:time_gracea}
GRACE-A runs in $O\bigl(n p (m d_e + p)\bigr)$ time, linear in $n$ for fixed $p$.
\end{theorem}
\begin{proof}
The anchor approximation eliminates the cubic and quadratic terms of \hyperref[thm:time_grace]{Theorem~\ref{thm:time_grace}}. Farthest-point selection of the $p$ anchors and construction of the sample-to-anchor affinity $\mathbf{Z} \in \mathbb{R}^{n \times p}$, which connects each sample to the anchors alone, together cost $O(n p m d_e)$ in place of the $O(n^2 m d_e)$ dense affinity. A truncated factorization replaces the full eigendecomposition at $O(n p^2)$, and the closing $k$-means contributes $O(t n K^2)$, both staying within the same order for fixed $p$. Collecting the terms yields $O\bigl(n p (m d_e + p)\bigr)$, which grows linearly in $n$ because $p$ stays bounded independently of $n$.
\end{proof}

\begin{theorem}
\label{thm:space_gracea}
GRACE-A requires $O\bigl(n (m d_e + p)\bigr)$ space.
\end{theorem}
\begin{proof}
The same anchoring contracts the footprint. The sample-to-anchor affinity $\mathbf{Z}$ holds $O(np)$ instead of the $O(n^2)$ dense matrix of \hyperref[thm:space_grace]{Theorem~\ref{thm:space_grace}}, and the value embeddings together with the input table account for the remaining $O(n m d_e)$. Because no retained structure scales with $n^2$, the dominant term is $O\bigl(n (m d_e + p)\bigr)$, again linear in $n$.
\end{proof}

Taken together, the four bounds capture the trade-off between the exact model and its scalable variant. Standard GRACE computes the knowledge-grounded affinity exactly, yet its cubic time and quadratic space confine it to datasets of moderate size. GRACE-A lifts both restrictions at once, reducing the time to $O\bigl(n p (m d_e + p)\bigr)$ and the space to $O\bigl(n (m d_e + p)\bigr)$, each linear in $n$ for a fixed anchor budget. The anchor design thereby carries the semantic metric of GRACE to corpora whose scale puts a dense $n \times n$ affinity out of reach, where storing it alone would exhaust commodity memory. As the experiments demonstrate the accuracy traded for the linear cost remains marginal, leaving GRACE-A with both the world-knowledge benefit of GRACE and the scalability to cluster tables far beyond the reach of dense spectral methods.

\begin{table}[t]
\centering
\renewcommand{\arraystretch}{0.8}
\small
\caption{\small{Statistics of the twenty-two datasets. $d_c$, $d_n$, and $n$ denote the numbers of categorical attributes, numerical attributes, and samples, respectively. $V$ is the total number of distinct values across all categorical attributes. $K$ equals the true number of clusters.}}
\label{tab:datasets}
\setlength{\tabcolsep}{3.5pt}
\begin{tabular}{c|c c|c c c c}
\toprule
No. & Dataset & Abbr. & $(d_c,\, d_n)$ & $n$ & $V$ & $K$ \\
\midrule
 1 & Lenses                & LE & $(4,\, 0)$   &      24 &   9 &  3 \\
\rowcolor{rowgray}
 2 & Caesarian Section     & CS & $(3,\, 2)$   &      80 &   8 &  2 \\
 3 & Zoo                   & ZO & $(16,\, 0)$  &     101 &  36 &  7 \\
\rowcolor{rowgray}
 4 & Autism Adolescent     & AA & $(18,\, 2)$  &     104 &  76 &  2 \\
 5 & Lymphography          & LY & $(18,\, 0)$  &     148 &  60 &  4 \\
\rowcolor{rowgray}
 6 & Teaching Assistant    & TA & $(4,\, 1)$   &     151 &  55 &  3 \\
 7 & Amphibians            & AM & $(12,\, 3)$  &     189 &  71 &  2 \\
\rowcolor{rowgray}
 8 & Soybean (Large)       & SO & $(35,\, 0)$  &     266 & 100 & 15 \\
 9 & SPECT Heart           & SH & $(22,\, 0)$  &     267 &  44 &  2 \\
\rowcolor{rowgray}
10 & Breast Cancer         & BC & $(9,\, 0)$   &     286 &  51 &  2 \\
11 & Heart Disease         & HD & $(8,\, 5)$   &     303 &  23 &  5 \\
\rowcolor{rowgray}
12 & Primary Tumor         & PT & $(17,\, 0)$  &     339 &  37 & 22 \\
13 & Dermatology           & DE & $(33,\, 1)$  &     366 & 130 &  6 \\
\rowcolor{rowgray}
14 & Chronic Kidney Disease & CK & $(13,\, 11)$ &     400 &  37 &  2 \\
15 & Congressional Voting  & CV & $(16,\, 0)$  &     435 &  32 &  2 \\
\rowcolor{rowgray}
16 & Statlog Australian    & SA & $(8,\, 6)$   &     690 &  37 &  2 \\
17 & Car Evaluation        & CA & $(6,\, 0)$  & 1{,}728 &  21 &  4 \\
\rowcolor{rowgray}
18 & Auction Verification  & AV & $(7,\, 1)$  & 2{,}043 &  54 &  2 \\
19 & Obesity Levels        & OL & $(8,\, 8)$   & 2{,}111 &  23 &  7 \\
\rowcolor{rowgray}
20 & Splice                & SP & $(60,\, 0)$  & 3{,}190 & 240 &  3 \\
21 & Mushroom              & MU & $(21,\, 0)$  & 8{,}124 & 115 &  2 \\
\rowcolor{rowgray}
22 & Adult                 & AD & $(8,\, 6)$   & 48{,}842 &  99 &  2 \\
\bottomrule
\end{tabular}
\end{table}

\begin{table*}[!t]
\scriptsize
\centering
\caption{\small{Clustering results on 11 categorical datasets and two reconstructed datasets (\textcolor{leakred}{in dark red}) for probing LLM label leakage, reported as mean $\pm$ standard deviation. \colorbox{best}{\textbf{Best}} and \colorbox{second}{second-best} results are highlighted (GRACE-A included). $\overline{AR}$ is the average rank over all 13 datasets; GRACE and GRACE-A are each ranked against the competitors only, not against each other. $^\dagger$ indicates statistical significance ($p<0.05$) against GRACE via the Wilcoxon signed-rank test.}}
\label{tab:pure_cat_results}
\resizebox{\textwidth}{!}{%
\renewcommand{\arraystretch}{0.8}
\begin{tabular}{c|c|cccccccccc|cc}
\toprule
\multirow{2}{*}{\textbf{Metric}} & \multirow{2}{*}{\textbf{Dataset}}
& KMO & ADC & GUDMM-S & MCDC & COForest & AMPHM & SigDT & HARR & DiSC & OCL & \textbf{GRACE} & \textbf{GRACE-A} \\
& & {\tiny[DMKD'98]} & {\tiny[TNNLS'23]} & {\tiny[PR'23]} & {\tiny[ICDCS'24]} & {\tiny[ECAI'24]} & {\tiny[CAIS'25]} & {\tiny[Inf.Sci.'25]} & {\tiny[ESWA'25]} & {\tiny[AAAI'26]} & {\tiny[SIGMOD'26]} & {\tiny\textbf{[Ours]}} & {\tiny\textbf{[Approx.]}} \\
\midrule
\multirow{11}{*}{\rotatebox[origin=c]{90}{\textbf{ARI}}}
 & LE & 0.043$_{\pm.06}$ & 0.061$_{\pm.10}$ & 0.130$_{\pm.09}$ & -0.076$_{\pm.00}$ & 0.119$_{\pm.11}$ & -0.013$_{\pm.00}$ & -0.043$_{\pm.00}$ & 0.125$_{\pm.16}$ & 0.065$_{\pm.18}$ & 0.137$_{\pm.10}$ & \cellcolor{second}0.163$_{\pm.17}$ & \cellcolor{best}\textbf{0.213$_{\pm.10}$} \\
 & ZO & 0.584$_{\pm.10}$ & 0.494$_{\pm.00}$ & 0.602$_{\pm.04}$ & 0.602$_{\pm.00}$ & 0.639$_{\pm.13}$ & 0.379$_{\pm.00}$ & 0.696$_{\pm.00}$ & \cellcolor{second}0.778$_{\pm.13}$ & 0.644$_{\pm.25}$ & 0.721$_{\pm.12}$ & \cellcolor{best}\textbf{0.827$_{\pm.00}$} & 0.660$_{\pm.00}$ \\
 & LY & 0.011$_{\pm.02}$ & 0.112$_{\pm.04}$ & 0.117$_{\pm.03}$ & -0.010$_{\pm.00}$ & 0.118$_{\pm.03}$ & 0.129$_{\pm.00}$ & 0.015$_{\pm.00}$ & 0.152$_{\pm.05}$ & 0.080$_{\pm.04}$ & 0.159$_{\pm.05}$ & \cellcolor{best}\textbf{0.215$_{\pm.01}$} & \cellcolor{second}0.161$_{\pm.00}$ \\
 & SO & 0.227$_{\pm.04}$ & 0.358$_{\pm.02}$ & 0.353$_{\pm.03}$ & 0.340$_{\pm.00}$ & 0.397$_{\pm.04}$ & 0.157$_{\pm.00}$ & 0.120$_{\pm.00}$ & 0.355$_{\pm.03}$ & 0.293$_{\pm.15}$ & 0.409$_{\pm.02}$ & \cellcolor{second}0.454$_{\pm.00}$ & \cellcolor{best}\textbf{0.456$_{\pm.00}$} \\
 & SH & -0.073$_{\pm.01}$ & -0.007$_{\pm.02}$ & -0.023$_{\pm.00}$ & -0.039$_{\pm.00}$ & -0.048$_{\pm.03}$ & 0.063$_{\pm.00}$ & 0.094$_{\pm.00}$ & 0.023$_{\pm.00}$ & \cellcolor{best}\textbf{0.312$_{\pm.00}$} & -0.088$_{\pm.00}$ & \cellcolor{second}0.294$_{\pm.00}$ & 0.143$_{\pm.00}$ \\
 & BC & 0.000$_{\pm.00}$ & 0.003$_{\pm.01}$ & 0.053$_{\pm.00}$ & -0.008$_{\pm.00}$ & 0.013$_{\pm.05}$ & 0.013$_{\pm.00}$ & 0.058$_{\pm.00}$ & 0.000$_{\pm.00}$ & 0.036$_{\pm.08}$ & -0.003$_{\pm.00}$ & \cellcolor{best}\textbf{0.156$_{\pm.00}$} & \cellcolor{second}0.147$_{\pm.00}$ \\
 & PT & 0.068$_{\pm.02}$ & 0.079$_{\pm.01}$ & 0.072$_{\pm.01}$ & 0.076$_{\pm.00}$ & \cellcolor{second}0.111$_{\pm.01}$ & 0.084$_{\pm.00}$ & 0.078$_{\pm.00}$ & 0.103$_{\pm.01}$ & 0.085$_{\pm.02}$ & 0.101$_{\pm.01}$ & \cellcolor{best}\textbf{0.115$_{\pm.00}$} & 0.110$_{\pm.00}$ \\
 & CV & 0.422$_{\pm.02}$ & 0.495$_{\pm.17}$ & 0.550$_{\pm.00}$ & 0.537$_{\pm.00}$ & 0.566$_{\pm.00}$ & 0.008$_{\pm.00}$ & 0.459$_{\pm.00}$ & 0.550$_{\pm.00}$ & 0.563$_{\pm.00}$ & \cellcolor{second}0.590$_{\pm.05}$ & \cellcolor{best}\textbf{0.628$_{\pm.00}$} & 0.573$_{\pm.00}$ \\
 & CA & 0.021$_{\pm.03}$ & 0.030$_{\pm.03}$ & 0.050$_{\pm.05}$ & 0.001$_{\pm.00}$ & 0.044$_{\pm.05}$ & \cellcolor{best}\textbf{0.194$_{\pm.00}$} & -0.041$_{\pm.00}$ & -0.002$_{\pm.08}$ & 0.032$_{\pm.08}$ & 0.048$_{\pm.04}$ & \cellcolor{second}0.131$_{\pm.00}$ & 0.087$_{\pm.00}$ \\
 & SP & 0.025$_{\pm.00}$ & 0.078$_{\pm.00}$ & 0.075$_{\pm.00}$ & 0.071$_{\pm.00}$ & 0.013$_{\pm.01}$ & -- & 0.386$_{\pm.00}$ & 0.107$_{\pm.08}$ & 0.000$_{\pm.00}$ & 0.078$_{\pm.01}$ & \cellcolor{second}0.396$_{\pm.00}$ & \cellcolor{best}\textbf{0.438$_{\pm.00}$} \\
 & MU & 0.373$_{\pm.05}$ & 0.344$_{\pm.20}$ & 0.273$_{\pm.00}$ & 0.066$_{\pm.01}$ & 0.111$_{\pm.10}$ & -- & 0.266$_{\pm.14}$ & 0.384$_{\pm.27}$ & 0.261$_{\pm.23}$ & 0.435$_{\pm.26}$ & \cellcolor{best}\textbf{0.606$_{\pm.01}$} & \cellcolor{second}0.581$_{\pm.00}$ \\
& \textcolor{leakred}{BC} & 0.020$_{\pm.00}$ & 0.002$_{\pm.00}$ & 0.054$_{\pm.00}$ & 0.003$_{\pm.00}$ & 0.006$_{\pm.00}$ & 0.013$_{\pm.00}$ & 0.011$_{\pm.00}$ & -0.004$_{\pm.00}$ & 0.047$_{\pm.00}$ & 0.028$_{\pm.00}$ & \cellcolor{best}\textbf{0.065$_{\pm.00}$} & \cellcolor{second}0.057$_{\pm.00}$ \\
 & \textcolor{leakred}{LY} & 0.029$_{\pm.00}$ & 0.051$_{\pm.00}$ & 0.045$_{\pm.00}$ & 0.025$_{\pm.00}$ & 0.029$_{\pm.00}$ & 0.038$_{\pm.00}$ & 0.052$_{\pm.00}$ & 0.020$_{\pm.00}$ & 0.050$_{\pm.00}$ & 0.025$_{\pm.00}$ & \cellcolor{best}\textbf{0.059$_{\pm.00}$} & \cellcolor{second}0.054$_{\pm.00}$ \\
\midrule
\multicolumn{2}{c|}{$\overline{AR}$} & 8.23$^\dagger$ & 6.54$^\dagger$ & 5.31$^\dagger$ & 9.00$^\dagger$ & 5.93$^\dagger$ & 6.73$^\dagger$ & 6.54$^\dagger$ & 5.77$^\dagger$ & 5.54$^\dagger$ & 4.61$^\dagger$ & \textbf{1.15} & \textbf{1.54} \\
\midrule
\multirow{11}{*}{\rotatebox[origin=c]{90}{\textbf{NMI}}}
 & LE & 0.165$_{\pm.06}$ & 0.232$_{\pm.11}$ & 0.210$_{\pm.10}$ & 0.200$_{\pm.00}$ & 0.287$_{\pm.15}$ & 0.129$_{\pm.00}$ & 0.027$_{\pm.00}$ & 0.183$_{\pm.16}$ & 0.226$_{\pm.15}$ & 0.287$_{\pm.10}$ & \cellcolor{best}\textbf{0.341$_{\pm.13}$} & \cellcolor{second}0.339$_{\pm.11}$ \\
 & ZO & 0.679$_{\pm.05}$ & 0.702$_{\pm.00}$ & 0.745$_{\pm.02}$ & 0.803$_{\pm.00}$ & 0.793$_{\pm.05}$ & 0.601$_{\pm.00}$ & 0.786$_{\pm.00}$ & \cellcolor{best}\textbf{0.869$_{\pm.04}$} & 0.736$_{\pm.20}$ & 0.810$_{\pm.04}$ & \cellcolor{second}0.854$_{\pm.00}$ & 0.819$_{\pm.00}$ \\
 & LY & 0.117$_{\pm.03}$ & 0.156$_{\pm.02}$ & 0.212$_{\pm.02}$ & 0.108$_{\pm.00}$ & 0.161$_{\pm.03}$ & \cellcolor{best}\textbf{0.330$_{\pm.00}$} & 0.141$_{\pm.00}$ & 0.174$_{\pm.05}$ & 0.173$_{\pm.05}$ & 0.196$_{\pm.05}$ & \cellcolor{second}0.253$_{\pm.00}$ & 0.218$_{\pm.00}$ \\
 & SO & 0.518$_{\pm.03}$ & 0.638$_{\pm.01}$ & 0.644$_{\pm.02}$ & 0.634$_{\pm.00}$ & 0.697$_{\pm.02}$ & 0.562$_{\pm.00}$ & 0.538$_{\pm.00}$ & 0.642$_{\pm.02}$ & 0.569$_{\pm.28}$ & 0.704$_{\pm.02}$ & \cellcolor{best}\textbf{0.756$_{\pm.00}$} & \cellcolor{second}0.731$_{\pm.00}$ \\
 & SH & 0.059$_{\pm.00}$ & 0.011$_{\pm.03}$ & 0.068$_{\pm.00}$ & 0.062$_{\pm.00}$ & 0.069$_{\pm.02}$ & 0.061$_{\pm.00}$ & 0.119$_{\pm.00}$ & 0.115$_{\pm.00}$ & \cellcolor{best}\textbf{0.185$_{\pm.00}$} & 0.094$_{\pm.00}$ & \cellcolor{second}0.144$_{\pm.00}$ & 0.135$_{\pm.00}$ \\
 & BC & 0.002$_{\pm.00}$ & 0.005$_{\pm.00}$ & 0.041$_{\pm.00}$ & 0.001$_{\pm.00}$ & 0.009$_{\pm.02}$ & 0.030$_{\pm.00}$ & 0.053$_{\pm.00}$ & 0.000$_{\pm.00}$ & 0.033$_{\pm.04}$ & 0.001$_{\pm.00}$ & \cellcolor{best}\textbf{0.079$_{\pm.00}$} & \cellcolor{second}0.076$_{\pm.00}$ \\
 & PT & 0.250$_{\pm.01}$ & 0.263$_{\pm.02}$ & 0.304$_{\pm.01}$ & 0.322$_{\pm.00}$ & 0.360$_{\pm.01}$ & 0.308$_{\pm.00}$ & 0.304$_{\pm.00}$ & 0.351$_{\pm.01}$ & 0.345$_{\pm.02}$ & 0.354$_{\pm.01}$ & \cellcolor{best}\textbf{0.370$_{\pm.00}$} & \cellcolor{second}0.364$_{\pm.00}$ \\
 & CV & 0.438$_{\pm.00}$ & 0.436$_{\pm.15}$ & 0.444$_{\pm.00}$ & 0.475$_{\pm.00}$ & 0.491$_{\pm.00}$ & 0.006$_{\pm.00}$ & \cellcolor{second}0.526$_{\pm.00}$ & 0.479$_{\pm.00}$ & 0.446$_{\pm.00}$ & 0.503$_{\pm.05}$ & \cellcolor{best}\textbf{0.538$_{\pm.00}$} & 0.483$_{\pm.00}$ \\
 & CA & 0.041$_{\pm.02}$ & 0.049$_{\pm.03}$ & 0.054$_{\pm.07}$ & 0.003$_{\pm.00}$ & 0.111$_{\pm.06}$ & 0.124$_{\pm.00}$ & 0.034$_{\pm.00}$ & 0.065$_{\pm.06}$ & 0.084$_{\pm.07}$ & 0.104$_{\pm.04}$ & \cellcolor{best}\textbf{0.188$_{\pm.00}$} & \cellcolor{second}0.133$_{\pm.00}$ \\
 & SP & 0.050$_{\pm.00}$ & 0.091$_{\pm.00}$ & 0.076$_{\pm.00}$ & 0.082$_{\pm.00}$ & 0.015$_{\pm.01}$ & -- & 0.329$_{\pm.00}$ & 0.115$_{\pm.09}$ & 0.000$_{\pm.00}$ & 0.128$_{\pm.03}$ & \cellcolor{best}\textbf{0.404$_{\pm.00}$} & \cellcolor{second}0.399$_{\pm.00}$ \\
 & MU & 0.329$_{\pm.02}$ & 0.361$_{\pm.17}$ & 0.303$_{\pm.00}$ & 0.052$_{\pm.01}$ & 0.098$_{\pm.10}$ & -- & 0.332$_{\pm.11}$ & 0.323$_{\pm.23}$ & 0.272$_{\pm.19}$ & 0.377$_{\pm.23}$ & \cellcolor{best}\textbf{0.554$_{\pm.01}$} & \cellcolor{second}0.553$_{\pm.00}$ \\
& \textcolor{leakred}{BC} & 0.110$_{\pm.00}$ & 0.062$_{\pm.00}$ & 0.138$_{\pm.00}$ & 0.099$_{\pm.00}$ & 0.111$_{\pm.00}$ & 0.090$_{\pm.00}$ & 0.136$_{\pm.00}$ & 0.080$_{\pm.00}$ & 0.129$_{\pm.00}$ & 0.144$_{\pm.00}$ & \cellcolor{second}0.149$_{\pm.00}$ & \cellcolor{best}\textbf{0.155$_{\pm.00}$} \\
 & \textcolor{leakred}{LY} & 0.179$_{\pm.00}$ & \cellcolor{best}\textbf{0.234$_{\pm.00}$} & \cellcolor{second}0.227$_{\pm.00}$ & 0.191$_{\pm.00}$ & 0.206$_{\pm.00}$ & 0.172$_{\pm.00}$ & 0.209$_{\pm.00}$ & 0.185$_{\pm.00}$ & 0.213$_{\pm.00}$ & 0.186$_{\pm.00}$ & 0.226$_{\pm.00}$ & \cellcolor{second}0.227$_{\pm.00}$ \\
\midrule
\multicolumn{2}{c|}{$\overline{AR}$} & 9.00$^\dagger$ & 7.15$^\dagger$ & 5.54$^\dagger$ & 7.69$^\dagger$ & 5.23$^\dagger$ & 7.73$^\dagger$ & 5.92$^\dagger$ & 6.00$^\dagger$ & 5.85$^\dagger$ & 4.00$^\dagger$ & \textbf{1.38} & \textbf{1.58} \\
\midrule
\multirow{11}{*}{\rotatebox[origin=c]{90}{\textbf{ACC}}}
 & LE & 0.517$_{\pm.07}$ & 0.529$_{\pm.06}$ & 0.545$_{\pm.09}$ & 0.417$_{\pm.00}$ & 0.554$_{\pm.09}$ & 0.458$_{\pm.00}$ & 0.417$_{\pm.00}$ & 0.567$_{\pm.10}$ & 0.500$_{\pm.12}$ & 0.558$_{\pm.10}$ & \cellcolor{second}0.575$_{\pm.11}$ & \cellcolor{best}\textbf{0.600$_{\pm.11}$} \\
 & ZO & 0.655$_{\pm.09}$ & 0.545$_{\pm.00}$ & 0.526$_{\pm.01}$ & 0.673$_{\pm.00}$ & 0.702$_{\pm.09}$ & 0.554$_{\pm.00}$ & \cellcolor{second}0.802$_{\pm.00}$ & 0.788$_{\pm.11}$ & 0.717$_{\pm.15}$ & 0.746$_{\pm.10}$ & \cellcolor{best}\textbf{0.822$_{\pm.00}$} & 0.756$_{\pm.00}$ \\
 & LY & 0.410$_{\pm.03}$ & 0.433$_{\pm.04}$ & 0.443$_{\pm.03}$ & 0.446$_{\pm.00}$ & 0.508$_{\pm.04}$ & \cellcolor{best}\textbf{0.561$_{\pm.00}$} & 0.493$_{\pm.00}$ & 0.487$_{\pm.04}$ & 0.470$_{\pm.07}$ & 0.513$_{\pm.04}$ & \cellcolor{second}0.546$_{\pm.00}$ & 0.509$_{\pm.00}$ \\
 & SO & 0.424$_{\pm.06}$ & 0.483$_{\pm.02}$ & 0.452$_{\pm.02}$ & 0.564$_{\pm.00}$ & 0.602$_{\pm.04}$ & 0.436$_{\pm.00}$ & 0.323$_{\pm.00}$ & 0.510$_{\pm.03}$ & 0.473$_{\pm.16}$ & 0.595$_{\pm.04}$ & \cellcolor{best}\textbf{0.658$_{\pm.00}$} & \cellcolor{second}0.616$_{\pm.00}$ \\
 & SH & 0.430$_{\pm.02}$ & 0.766$_{\pm.08}$ & 0.547$_{\pm.00}$ & 0.502$_{\pm.00}$ & 0.544$_{\pm.02}$ & 0.629$_{\pm.00}$ & 0.397$_{\pm.00}$ & 0.592$_{\pm.00}$ & \cellcolor{best}\textbf{0.843$_{\pm.00}$} & 0.559$_{\pm.00}$ & \cellcolor{second}0.824$_{\pm.00}$ & 0.793$_{\pm.00}$ \\
 & BC & 0.516$_{\pm.02}$ & 0.544$_{\pm.01}$ & 0.575$_{\pm.00}$ & 0.531$_{\pm.00}$ & 0.531$_{\pm.07}$ & 0.567$_{\pm.00}$ & 0.523$_{\pm.00}$ & \cellcolor{second}0.708$_{\pm.00}$ & 0.609$_{\pm.08}$ & 0.509$_{\pm.00}$ & \cellcolor{best}\textbf{0.711$_{\pm.00}$} & 0.704$_{\pm.00}$ \\
 & PT & 0.191$_{\pm.02}$ & 0.280$_{\pm.02}$ & 0.287$_{\pm.01}$ & 0.265$_{\pm.00}$ & \cellcolor{best}\textbf{0.306$_{\pm.02}$} & 0.286$_{\pm.00}$ & 0.256$_{\pm.00}$ & 0.287$_{\pm.02}$ & 0.279$_{\pm.02}$ & \cellcolor{second}0.299$_{\pm.02}$ & 0.293$_{\pm.00}$ & 0.287$_{\pm.00}$ \\
 & CV & 0.762$_{\pm.01}$ & 0.846$_{\pm.08}$ & 0.871$_{\pm.00}$ & 0.867$_{\pm.00}$ & 0.877$_{\pm.00}$ & 0.618$_{\pm.00}$ & 0.657$_{\pm.00}$ & 0.871$_{\pm.00}$ & 0.876$_{\pm.00}$ & \cellcolor{second}0.884$_{\pm.02}$ & \cellcolor{best}\textbf{0.897$_{\pm.00}$} & 0.879$_{\pm.00}$ \\
 & CA & 0.262$_{\pm.03}$ & 0.393$_{\pm.05}$ & 0.375$_{\pm.04}$ & 0.270$_{\pm.00}$ & 0.383$_{\pm.05}$ & 0.514$_{\pm.00}$ & \cellcolor{second}0.554$_{\pm.00}$ & \cellcolor{best}\textbf{0.600$_{\pm.08}$} & 0.544$_{\pm.12}$ & 0.372$_{\pm.04}$ & 0.487$_{\pm.00}$ & 0.398$_{\pm.00}$ \\
 & SP & 0.377$_{\pm.00}$ & 0.513$_{\pm.00}$ & 0.510$_{\pm.00}$ & 0.556$_{\pm.00}$ & 0.392$_{\pm.02}$ & -- & 0.687$_{\pm.00}$ & 0.554$_{\pm.05}$ & 0.519$_{\pm.00}$ & 0.474$_{\pm.01}$ & \cellcolor{second}0.748$_{\pm.00}$ & \cellcolor{best}\textbf{0.776$_{\pm.00}$} \\
 & MU & 0.678$_{\pm.02}$ & 0.782$_{\pm.08}$ & 0.678$_{\pm.00}$ & 0.628$_{\pm.01}$ & 0.650$_{\pm.07}$ & -- & 0.551$_{\pm.04}$ & 0.759$_{\pm.17}$ & 0.732$_{\pm.11}$ & 0.806$_{\pm.12}$ & \cellcolor{best}\textbf{0.889$_{\pm.00}$} & \cellcolor{second}0.813$_{\pm.00}$ \\
& \textcolor{leakred}{BC} & 0.206$_{\pm.00}$ & 0.191$_{\pm.00}$ & 0.219$_{\pm.00}$ & 0.188$_{\pm.00}$ & 0.217$_{\pm.00}$ & 0.206$_{\pm.00}$ & 0.213$_{\pm.00}$ & 0.191$_{\pm.00}$ & 0.227$_{\pm.00}$ & 0.217$_{\pm.00}$ & \cellcolor{best}\textbf{0.246$_{\pm.00}$} & \cellcolor{second}0.238$_{\pm.00}$ \\
 & \textcolor{leakred}{LY} & 0.247$_{\pm.00}$ & 0.251$_{\pm.00}$ & 0.254$_{\pm.00}$ & 0.243$_{\pm.00}$ & 0.257$_{\pm.00}$ & 0.257$_{\pm.00}$ & 0.267$_{\pm.00}$ & 0.250$_{\pm.00}$ & \cellcolor{best}\textbf{0.277$_{\pm.00}$} & 0.250$_{\pm.00}$ & \cellcolor{second}0.270$_{\pm.00}$ & \cellcolor{second}0.270$_{\pm.00}$ \\
\midrule
\multicolumn{2}{c|}{$\overline{AR}$} & 9.31$^\dagger$ & 6.73$^\dagger$ & 6.50$^\dagger$ & 8.16$^\dagger$ & 5.30$^\dagger$ & 6.36$^\dagger$ & 7.08$^\dagger$ & 4.69$^\dagger$ & 4.54$^\dagger$ & 5.08$^\dagger$ & \textbf{1.69} & \textbf{2.15} \\
\bottomrule
\end{tabular}%
}
\end{table*}

\begin{table*}[!t]
\centering
\caption{Clustering results on 11 mixed datasets. All notations are consistent with those defined in \hyperref[tab:pure_cat_results]{Table~\ref{tab:pure_cat_results}}.}
\label{tab:mixed_results}
\renewcommand{\arraystretch}{0.8}
\begin{tabular}{c|c|ccccc|cc}
\toprule
\multirow{2}{*}{\textbf{Metric}} & \multirow{2}{*}{\textbf{Dataset}} & KPR & ADC & GUDMM-S & AMPHM & HARR & \textbf{GRACE} & \textbf{GRACE-A} \\
& & {\small[PAKDD'97]} & {\small[TNNLS'23]} & {\small[PR'23]} & {\small[CAIS'25]} & {\small[ESWA'25]} & {\small\textbf{[Ours]}} & {\small\textbf{[Approx.]}} \\
\midrule
\multirow{11}{*}{\rotatebox[origin=c]{90}{\textbf{ARI}}}
 & CS & 0.003$_{\pm.00}$ & -0.003$_{\pm.01}$ & 0.010$_{\pm.00}$ & 0.050$_{\pm.00}$ & 0.044$_{\pm.04}$ & \cellcolor{best}\textbf{0.094$_{\pm.00}$} & \cellcolor{second}0.078$_{\pm.00}$ \\
 & AA & 0.478$_{\pm.02}$ & 0.029$_{\pm.02}$ & 0.509$_{\pm.00}$ & 0.040$_{\pm.00}$ & 0.131$_{\pm.23}$ & \cellcolor{second}0.528$_{\pm.00}$ & \cellcolor{best}\textbf{0.588$_{\pm.00}$} \\
 & TA & 0.004$_{\pm.01}$ & 0.012$_{\pm.01}$ & 0.016$_{\pm.00}$ & 0.016$_{\pm.00}$ & -0.008$_{\pm.01}$ & \cellcolor{best}\textbf{0.065$_{\pm.00}$} & \cellcolor{second}0.052$_{\pm.00}$ \\
 & AM & -0.005$_{\pm.00}$ & 0.008$_{\pm.00}$ & 0.067$_{\pm.00}$ & 0.010$_{\pm.00}$ & 0.001$_{\pm.01}$ & \cellcolor{best}\textbf{0.109$_{\pm.00}$} & \cellcolor{second}0.090$_{\pm.00}$ \\
 & HD & 0.019$_{\pm.00}$ & 0.229$_{\pm.00}$ & 0.372$_{\pm.00}$ & 0.242$_{\pm.00}$ & 0.337$_{\pm.00}$ & \cellcolor{best}\textbf{0.394$_{\pm.00}$} & \cellcolor{second}0.390$_{\pm.00}$ \\
 & DE & 0.101$_{\pm.01}$ & 0.351$_{\pm.08}$ & 0.705$_{\pm.00}$ & 0.679$_{\pm.00}$ & 0.805$_{\pm.08}$ & \cellcolor{best}\textbf{0.921$_{\pm.00}$} & \cellcolor{second}0.914$_{\pm.00}$ \\
 & CK & 0.112$_{\pm.00}$ & 0.664$_{\pm.00}$ & 0.510$_{\pm.00}$ & 0.567$_{\pm.00}$ & 0.626$_{\pm.00}$ & \cellcolor{best}\textbf{0.895$_{\pm.00}$} & \cellcolor{second}0.815$_{\pm.00}$ \\
 & SA & 0.004$_{\pm.00}$ & 0.077$_{\pm.01}$ & 0.347$_{\pm.00}$ & -0.009$_{\pm.00}$ & \cellcolor{best}\textbf{0.417$_{\pm.01}$} & \cellcolor{second}0.387$_{\pm.00}$ & 0.363$_{\pm.00}$ \\
 & AV & -0.100$_{\pm.00}$ & \cellcolor{second}0.004$_{\pm.01}$ & 0.003$_{\pm.00}$ & -0.047$_{\pm.00}$ & -0.009$_{\pm.05}$ & \cellcolor{best}\textbf{0.053$_{\pm.00}$} & -0.002$_{\pm.00}$ \\
 & OL & 0.214$_{\pm.00}$ & 0.079$_{\pm.01}$ & 0.231$_{\pm.00}$ & 0.163$_{\pm.00}$ & 0.090$_{\pm.02}$ & \cellcolor{second}0.255$_{\pm.00}$ & \cellcolor{best}\textbf{0.264$_{\pm.00}$} \\
 & AD & -0.012$_{\pm.00}$ & 0.084$_{\pm.00}$ & 0.086$_{\pm.01}$ & -- & \cellcolor{second}0.100$_{\pm.00}$ & 0.093$_{\pm.03}$ & \cellcolor{best}\textbf{0.156$_{\pm.01}$} \\
\midrule
\multicolumn{2}{c|}{$\overline{AR}$} & 5.09$^\dagger$ & 4.36$^\dagger$ & 2.82$^\dagger$ & 4.00$^\dagger$ & 3.36$^\dagger$ & \textbf{1.18} & \textbf{1.27} \\
\midrule
\multirow{11}{*}{\rotatebox[origin=c]{90}{\textbf{NMI}}}
 & CS & 0.018$_{\pm.00}$ & 0.006$_{\pm.01}$ & 0.017$_{\pm.00}$ & 0.077$_{\pm.00}$ & 0.053$_{\pm.03}$ & \cellcolor{best}\textbf{0.108$_{\pm.00}$} & \cellcolor{second}0.098$_{\pm.00}$ \\
 & AA & 0.347$_{\pm.01}$ & 0.019$_{\pm.01}$ & 0.432$_{\pm.00}$ & 0.058$_{\pm.00}$ & 0.122$_{\pm.22}$ & \cellcolor{second}0.460$_{\pm.00}$ & \cellcolor{best}\textbf{0.492$_{\pm.00}$} \\
 & TA & 0.014$_{\pm.00}$ & 0.021$_{\pm.01}$ & 0.046$_{\pm.00}$ & \cellcolor{best}\textbf{0.084$_{\pm.00}$} & 0.004$_{\pm.01}$ & \cellcolor{second}0.074$_{\pm.00}$ & 0.065$_{\pm.00}$ \\
 & AM & 0.016$_{\pm.00}$ & 0.004$_{\pm.00}$ & 0.079$_{\pm.00}$ & 0.005$_{\pm.00}$ & 0.002$_{\pm.00}$ & \cellcolor{best}\textbf{0.113$_{\pm.00}$} & \cellcolor{second}0.100$_{\pm.00}$ \\
 & HD & 0.014$_{\pm.00}$ & 0.174$_{\pm.00}$ & 0.294$_{\pm.01}$ & 0.189$_{\pm.00}$ & 0.262$_{\pm.00}$ & \cellcolor{second}0.306$_{\pm.00}$ & \cellcolor{best}\textbf{0.308$_{\pm.00}$} \\
 & DE & 0.218$_{\pm.02}$ & 0.403$_{\pm.08}$ & 0.749$_{\pm.00}$ & 0.740$_{\pm.00}$ & 0.822$_{\pm.06}$ & \cellcolor{best}\textbf{0.913$_{\pm.00}$} & \cellcolor{second}0.904$_{\pm.00}$ \\
 & CK & 0.246$_{\pm.00}$ & 0.639$_{\pm.00}$ & 0.528$_{\pm.00}$ & 0.568$_{\pm.00}$ & 0.611$_{\pm.00}$ & \cellcolor{best}\textbf{0.848$_{\pm.00}$} & \cellcolor{second}0.767$_{\pm.00}$ \\
 & SA & 0.016$_{\pm.00}$ & 0.059$_{\pm.00}$ & 0.288$_{\pm.00}$ & 0.040$_{\pm.00}$ & \cellcolor{best}\textbf{0.323$_{\pm.01}$} & \cellcolor{second}0.300$_{\pm.00}$ & 0.278$_{\pm.00}$ \\
 & AV & 0.003$_{\pm.00}$ & 0.004$_{\pm.00}$ & 0.003$_{\pm.00}$ & \cellcolor{second}0.009$_{\pm.00}$ & \cellcolor{best}\textbf{0.020$_{\pm.03}$} & 0.006$_{\pm.00}$ & 0.000$_{\pm.00}$ \\
 & OL & 0.281$_{\pm.00}$ & 0.135$_{\pm.02}$ & 0.338$_{\pm.00}$ & 0.249$_{\pm.00}$ & 0.146$_{\pm.02}$ & \cellcolor{second}0.346$_{\pm.00}$ & \cellcolor{best}\textbf{0.354$_{\pm.00}$} \\
 & AD & 0.001$_{\pm.00}$ & 0.077$_{\pm.00}$ & 0.057$_{\pm.01}$ & -- & \cellcolor{second}0.082$_{\pm.00}$ & 0.064$_{\pm.05}$ & \cellcolor{best}\textbf{0.125$_{\pm.01}$} \\
\midrule
\multicolumn{2}{c|}{$\overline{AR}$} & 4.73$^\dagger$ & 4.45$^\dagger$ & 3.36$^\dagger$ & 3.50$^\dagger$ & 3.18$^\dagger$ & \textbf{1.55} & \textbf{1.73} \\
\midrule
\multirow{11}{*}{\rotatebox[origin=c]{90}{\textbf{ACC}}}
 & CS & 0.562$_{\pm.00}$ & 0.540$_{\pm.03}$ & 0.575$_{\pm.00}$ & 0.625$_{\pm.00}$ & 0.606$_{\pm.05}$ & \cellcolor{best}\textbf{0.662$_{\pm.00}$} & \cellcolor{second}0.650$_{\pm.00}$ \\
 & AA & 0.784$_{\pm.01}$ & 0.592$_{\pm.04}$ & 0.828$_{\pm.00}$ & 0.612$_{\pm.00}$ & 0.631$_{\pm.14}$ & \cellcolor{second}0.867$_{\pm.00}$ & \cellcolor{best}\textbf{0.885$_{\pm.00}$} \\
 & TA & 0.383$_{\pm.01}$ & 0.426$_{\pm.02}$ & 0.404$_{\pm.00}$ & 0.384$_{\pm.00}$ & 0.352$_{\pm.02}$ & \cellcolor{best}\textbf{0.470$_{\pm.00}$} & \cellcolor{second}0.459$_{\pm.00}$ \\
 & AM & 0.561$_{\pm.00}$ & 0.566$_{\pm.00}$ & 0.635$_{\pm.00}$ & 0.571$_{\pm.00}$ & 0.554$_{\pm.01}$ & \cellcolor{best}\textbf{0.669$_{\pm.00}$} & \cellcolor{second}0.654$_{\pm.00}$ \\
 & HD & 0.576$_{\pm.00}$ & 0.741$_{\pm.00}$ & 0.806$_{\pm.00}$ & 0.748$_{\pm.00}$ & 0.791$_{\pm.00}$ & \cellcolor{best}\textbf{0.815$_{\pm.00}$} & \cellcolor{second}0.813$_{\pm.00}$ \\
 & DE & 0.363$_{\pm.02}$ & 0.494$_{\pm.08}$ & 0.746$_{\pm.00}$ & 0.768$_{\pm.00}$ & 0.832$_{\pm.07}$ & \cellcolor{best}\textbf{0.953$_{\pm.00}$} & \cellcolor{second}0.947$_{\pm.00}$ \\
 & CK & 0.669$_{\pm.00}$ & 0.908$_{\pm.00}$ & 0.858$_{\pm.00}$ & 0.877$_{\pm.00}$ & 0.896$_{\pm.00}$ & \cellcolor{best}\textbf{0.973$_{\pm.00}$} & \cellcolor{second}0.952$_{\pm.00}$ \\
 & SA & 0.562$_{\pm.00}$ & 0.642$_{\pm.01}$ & 0.792$_{\pm.00}$ & 0.520$_{\pm.00}$ & \cellcolor{best}\textbf{0.823$_{\pm.00}$} & \cellcolor{second}0.812$_{\pm.00}$ & 0.801$_{\pm.00}$ \\
 & AV & 0.660$_{\pm.00}$ & 0.552$_{\pm.03}$ & 0.527$_{\pm.00}$ & \cellcolor{best}\textbf{0.814$_{\pm.00}$} & 0.582$_{\pm.05}$ & \cellcolor{second}0.801$_{\pm.00}$ & 0.562$_{\pm.00}$ \\
 & OL & 0.324$_{\pm.00}$ & 0.274$_{\pm.01}$ & \cellcolor{second}0.429$_{\pm.00}$ & 0.390$_{\pm.00}$ & 0.283$_{\pm.02}$ & \cellcolor{best}\textbf{0.437$_{\pm.00}$} & 0.423$_{\pm.00}$ \\
 & AD & 0.613$_{\pm.01}$ & 0.654$_{\pm.00}$ & 0.690$_{\pm.01}$ & -- & 0.660$_{\pm.00}$ & \cellcolor{best}\textbf{0.739$_{\pm.04}$} & \cellcolor{second}0.701$_{\pm.00}$ \\
\midrule
\multicolumn{2}{c|}{$\overline{AR}$} & 4.82$^\dagger$ & 4.45$^\dagger$ & 3.18$^\dagger$ & 3.50$^\dagger$ & 3.64$^\dagger$ & \textbf{1.18} & \textbf{1.45} \\
\bottomrule
\end{tabular}%
\end{table*}

\begin{table}[!t]
\scriptsize
\centering
\caption{ARI of GRACE and GRACE-A versus LLM-based clustering methods on 22 datasets. All notations are consistent with those defined in \hyperref[tab:pure_cat_results]{Table~\ref{tab:pure_cat_results}}.}
\label{tab:llm_clustering_ari}
\resizebox{\columnwidth}{!}{%
\renewcommand{\arraystretch}{0.85}
\setlength{\tabcolsep}{4pt}
\begin{tabular}{c|ccccc|cc}
\toprule
\textbf{Dataset} & TabLLM & ClusterLLM & FewShot & GenericDesc. & BREVE & \textbf{GRACE} & \textbf{GRACE-A} \\
\midrule
LE & 0.026 & 0.000 & -0.043 & 0.094 & \cellcolor{best}\textbf{0.217} & 0.163 & \cellcolor{second}0.213 \\
CS & 0.003 & -0.014 & 0.064 & 0.064 & -- & \cellcolor{best}\textbf{0.094} & \cellcolor{second}0.078 \\
ZO & 0.631 & 0.000 & \cellcolor{second}0.805 & 0.447 & 0.774 & \cellcolor{best}\textbf{0.827} & 0.660 \\
AA & -0.004 & 0.074 & 0.024 & 0.466 & -- & \cellcolor{second}0.528 & \cellcolor{best}\textbf{0.588} \\
LY & 0.072 & 0.000 & 0.061 & 0.020 & \cellcolor{second}0.207 & \cellcolor{best}\textbf{0.215} & 0.161 \\
TA & 0.045 & 0.006 & 0.039 & -0.011 & -- & \cellcolor{best}\textbf{0.065} & \cellcolor{second}0.052 \\
AM & 0.022 & \cellcolor{second}0.110 & 0.022 & \cellcolor{best}\textbf{0.129} & -- & 0.109 & 0.090 \\
SO & 0.173 & 0.000 & 0.139 & 0.150 & 0.428 & \cellcolor{second}0.454 & \cellcolor{best}\textbf{0.456} \\
SH & -0.016 & 0.000 & -0.013 & -0.010 & -0.022 & \cellcolor{best}\textbf{0.294} & \cellcolor{second}0.143 \\
BC & 0.117 & 0.008 & 0.008 & -0.002 & \cellcolor{best}\textbf{0.169} & \cellcolor{second}0.156 & 0.147 \\
HD & 0.174 & -0.004 & 0.002 & 0.306 & -- & \cellcolor{best}\textbf{0.394} & \cellcolor{second}0.390 \\
PT & 0.056 & 0.000 & 0.058 & 0.036 & 0.103 & \cellcolor{best}\textbf{0.115} & \cellcolor{second}0.110 \\
DE & 0.159 & -0.005 & 0.104 & 0.314 & -- & \cellcolor{best}\textbf{0.921} & \cellcolor{second}0.914 \\
CK & 0.396 & 0.672 & 0.101 & 0.639 & -- & \cellcolor{best}\textbf{0.895} & \cellcolor{second}0.815 \\
CV & 0.006 & 0.000 & \cellcolor{second}0.606 & 0.061 & 0.578 & \cellcolor{best}\textbf{0.628} & 0.573 \\
SA & 0.069 & -0.004 & 0.000 & 0.193 & -- & \cellcolor{best}\textbf{0.387} & \cellcolor{second}0.363 \\
CA & 0.009 & 0.011 & \cellcolor{best}\textbf{0.169} & 0.007 & 0.085 & \cellcolor{second}0.131 & 0.087 \\
AV & 0.014 & \cellcolor{best}\textbf{0.286} & -0.002 & 0.003 & -- & \cellcolor{second}0.053 & -0.002 \\
OL & 0.222 & 0.131 & 0.142 & 0.130 & -- & \cellcolor{second}0.255 & \cellcolor{best}\textbf{0.264} \\
SP & 0.030 & 0.000 & 0.006 & 0.033 & 0.353 & \cellcolor{second}0.396 & \cellcolor{best}\textbf{0.438} \\
MU & 0.081 & 0.000 & \cellcolor{best}\textbf{0.648} & 0.093 & 0.593 & \cellcolor{second}0.606 & 0.581 \\
AD & \cellcolor{second}0.104 & 0.026 & 0.032 & 0.094 & -- & 0.093 & \cellcolor{best}\textbf{0.156} \\
\midrule
$\overline{AR}$ & 3.61$^\dagger$ & 4.48$^\dagger$ & 3.66$^\dagger$ & 3.57$^\dagger$ & 2.55$^\dagger$ & \textbf{1.41} & \textbf{1.73} \\
\bottomrule
\end{tabular}%
}
\end{table}

\section{Experiments}

\subsection{Experimental Settings}
\label{sec:settings}

\textbf{Four experimental studies and one case study} are conducted to evaluate GRACE from complementary perspectives. \textbf{1) Clustering effectiveness} (\hyperref[sec:performance]{\S\ref{sec:performance}}) compares GRACE with eleven representative counterparts on twenty UCI datasets and reports Wilcoxon signed-rank tests to assess statistical reliability. \textbf{2) Component analysis} (\hyperref[sec:ablation]{\S\ref{sec:ablation}}) isolates the contribution of semantic representation and dual-view neighborhood consistency. \textbf{3) Generality and robustness} (\hyperref[sec:representation]{\S\ref{sec:representation}}, \hyperref[sec:robustness]{\S\ref{sec:robustness}}) examines whether the learned representation remains effective across clustering algorithms and whether the results are stable across different LLM backends. \textbf{4) Scalability evaluation} (\hyperref[sec:scalability]{\S\ref{sec:scalability}}) studies runtime and clustering quality under increasing sample size, attribute cardinality, and cluster count. \textbf{5) Case study} (\hyperref[sec:case]{\S\ref{sec:case}}) analyzes the learned semantic metric on the CK dataset to illustrate how GRACE improves cluster separation. All methods run on an NVIDIA RTX 4090, except the scalability study, which is timed on a single commodity RTX 3060 (6GB) to provide an independent efficiency measurement.

\begin{table*}[t]
\small
\centering
\caption{Ablation study over Semantic Representation (SR) and Dual-View Neighborhood Consistency (DVNC).}
\label{tab:ablation}
\resizebox{\textwidth}{!}{%
\renewcommand{\arraystretch}{0.8}
\setlength{\tabcolsep}{2pt}
\begin{tabular}{c|cc|c|cccccccccccccccccccc|c}
\toprule
\multirow{2}{*}{\textbf{Metric}} & \multicolumn{2}{c|}{\scriptsize{\textbf{Components}}} & \multirow{2}{*}{\textbf{Mark}} & \multicolumn{20}{c|}{\textbf{Datasets}} & \multirow{2}{*}{$\overline{AR}$} \\
\cmidrule{2-3} \cmidrule{5-24}
 & \textbf{SR} & \textbf{DVNC} & & LE & CS & ZO & AA & LY & TA & AM & SO & SH & BC & HD & PT & DE & CK & CV & SA & CA & AV & OL & SP & \\
\midrule
 \multirow{3}{*}{\rotatebox[origin=c]{90}{\textbf{ARI}}} &  &  & I & 0.064 & 0.094 & 0.807 & 0.366 & 0.192 & 0.019 & 0.068 & 0.432 & 0.068 & \textbf{0.173} & 0.370 & 0.099 & 0.778 & 0.755 & 0.585 & 0.354 & 0.000 & 0.049 & \textbf{0.290} & 0.000 & 2.65 \\
  & $\checkmark$ &  & II & 0.064 & \textbf{0.111} & 0.656 & \textbf{0.529} & \textbf{0.238} & \textbf{0.065} & 0.086 & \textbf{0.456} & 0.068 & 0.151 & \textbf{0.420} & 0.099 & \textbf{0.921} & 0.768 & 0.599 & \textbf{0.387} & \textbf{0.131} & \textbf{0.053} & 0.241 & 0.379 & 1.85 \\
  & $\checkmark$ & $\checkmark$ & \textbf{Full} & \textbf{0.130} & 0.094 & \textbf{0.827} & 0.528 & 0.217 & \textbf{0.065} & \textbf{0.108} & 0.454 & \textbf{0.294} & 0.156 & 0.394 & \textbf{0.116} & \textbf{0.921} & \textbf{0.895} & \textbf{0.628} & \textbf{0.387} & \textbf{0.131} & \textbf{0.053} & 0.255 & \textbf{0.396} & \textbf{1.50} \\
\midrule
 \multirow{3}{*}{\rotatebox[origin=c]{90}{\textbf{NMI}}} &  &  & I & 0.137 & 0.078 & 0.803 & 0.264 & 0.243 & 0.026 & 0.066 & 0.725 & 0.119 & \textbf{0.089} & 0.280 & 0.353 & 0.851 & 0.713 & 0.505 & 0.280 & 0.003 & \textbf{0.034} & 0.354 & 0.001 & 2.65 \\
  & $\checkmark$ &  & II & 0.137 & \textbf{0.131} & 0.794 & 0.437 & \textbf{0.275} & \textbf{0.074} & 0.088 & 0.748 & 0.119 & 0.076 & \textbf{0.329} & 0.354 & \textbf{0.913} & 0.725 & 0.520 & \textbf{0.300} & \textbf{0.188} & 0.006 & \textbf{0.364} & 0.401 & 1.88 \\
  & $\checkmark$ & $\checkmark$ & \textbf{Full} & \textbf{0.337} & 0.108 & \textbf{0.854} & \textbf{0.460} & 0.255 & \textbf{0.074} & \textbf{0.112} & \textbf{0.756} & \textbf{0.144} & 0.079 & 0.306 & \textbf{0.371} & \textbf{0.913} & \textbf{0.848} & \textbf{0.538} & \textbf{0.300} & \textbf{0.188} & 0.006 & 0.346 & \textbf{0.404} & \textbf{1.48} \\
\midrule
 \multirow{3}{*}{\rotatebox[origin=c]{90}{\textbf{ACC}}} &  &  & I & \textbf{0.583} & 0.662 & 0.801 & 0.806 & 0.513 & 0.430 & 0.635 & 0.633 & 0.633 & \textbf{0.722} & 0.795 & 0.283 & 0.813 & 0.935 & 0.883 & 0.806 & 0.266 & 0.702 & \textbf{0.515} & 0.519 & 2.60 \\
  & $\checkmark$ &  & II & \textbf{0.583} & \textbf{0.675} & 0.673 & \textbf{0.867} & 0.540 & \textbf{0.470} & 0.651 & \textbf{0.680} & 0.633 & 0.708 & \textbf{0.825} & 0.282 & \textbf{0.953} & 0.939 & 0.887 & \textbf{0.812} & \textbf{0.487} & \textbf{0.801} & 0.427 & 0.737 & 1.93 \\
  & $\checkmark$ & $\checkmark$ & \textbf{Full} & \textbf{0.583} & 0.662 & \textbf{0.822} & \textbf{0.867} & \textbf{0.547} & \textbf{0.470} & \textbf{0.668} & 0.658 & \textbf{0.824} & 0.711 & 0.815 & \textbf{0.294} & \textbf{0.953} & \textbf{0.973} & \textbf{0.897} & \textbf{0.812} & \textbf{0.487} & \textbf{0.801} & 0.437 & \textbf{0.748} & \textbf{1.48} \\
\bottomrule
\end{tabular}
}
\end{table*}

\begin{table*}[!t]
\centering
\caption{Mechanism ablation on ten mixed-type datasets (ACC). The two components indicate whether semantic grounding ($\checkmark$) replaces the default distance for categorical (Cat., default Hamming) and numerical (Num., default Euclidean) attributes. Row~I (neither) reduces to KPR and Full (both) is the complete GRACE. $\overline{AR}$ is the average rank; the best result per dataset is in \textbf{bold}.}
\label{tab:mix_mechanism_ablation}
\renewcommand{\arraystretch}{0.95}
\begin{tabular}{c|cc|c|cccccccccc|c}
\toprule
\multirow{2}{*}{\textbf{Metric}} & \multicolumn{2}{c|}{\scriptsize{\textbf{Components}}} & \multirow{2}{*}{\textbf{Mark}} & \multicolumn{10}{c|}{\textbf{Datasets}} & \multirow{2}{*}{$\overline{AR}$} \\
\cmidrule{2-3}\cmidrule{5-14}
 & \textbf{Cat.} & \textbf{Num.} & & CS & AA & TA & AM & HD & DE & CK & SA & AV & OL & \\
\midrule
\multirow{4}{*}{\rotatebox[origin=c]{90}{\textbf{ACC}}}
 &              &              & I            & 0.562 & 0.784 & 0.383 & 0.561 & 0.576 & 0.363 & 0.669 & 0.562 & 0.660 & 0.324 & 3.60 \\
 &              & $\checkmark$ & II           & 0.599 & 0.663 & 0.348 & 0.646 & 0.754 & 0.377 & 0.915 & 0.801 & 0.595 & 0.383 & 2.90 \\
 & $\checkmark$ &              & III          & \textbf{0.665} & 0.853 & 0.458 & 0.626 & \textbf{0.827} & 0.623 & 0.865 & 0.794 & 0.509 & 0.394 & 2.30 \\
 & $\checkmark$ & $\checkmark$ & \textbf{Full} & 0.662 & \textbf{0.867} & \textbf{0.470} & \textbf{0.670} & 0.815 & \textbf{0.953} & \textbf{0.973} & \textbf{0.812} & \textbf{0.801} & \textbf{0.437} & \textbf{1.20} \\
\bottomrule
\end{tabular}%
\end{table*}

\textbf{Eleven representative clustering counterparts, five LLM-based baselines, and four LLM backends} are selected for comparison.
Partition-based methods include $k$-Modes (KMO)~\cite{huang1998extensions} and $k$-Prototypes (KPR)~\cite{huang1997clustering}.
Distance metric learning methods consist of DiSC~\cite{zhao2025break}, ADC~\cite{zhang2022graph}, HARR~\cite{zhang2025learning}, MCDC~\cite{cai2024robust}, and GUDMM-S~\cite{mousavi2023generalized}.
Tree-based approaches include COForest~\cite{zhao2024learning} and SigDT~\cite{hu2025significance}.
OCL~\cite{zhang2025categorical} adopts ordinalization, and AMPHM~\cite{zhang2025adaptive} employs multi-resolution condensation.
Among these counterparts, KPR, ADC, HARR, AMPHM, and GUDMM-S natively support mixed data.
The five LLM-based baselines adapt representative LLM-for-clustering paradigms to our mixed-attribute setting.
Following the \emph{row-serialization} paradigm of TabLLM~\cite{hegselmann2023tabllm}, we serialize each record into a natural-language string over its attribute names and values, encode it with $f_\theta$, and cluster the resulting embeddings.
Adapting the \emph{triplet-guided representation} idea of ClusterLLM~\cite{zhang2023clusterllm}, we elicit LLM triplet judgments over records to shape the clustering representation.
Following the \emph{few-shot pairwise-constraint} paradigm of~\cite{viswanathan2024large}, we derive LLM-generated must-link/cannot-link constraints and run constrained clustering.
\emph{GenericDescribe} prompts the LLM for a single free-form description per attribute value and clusters the description embeddings, serving as a description-only variant.
Finally, BREVE~\cite{yang2026bridging}, preliminary study on \emph{purely categorical} data, generates per-value descriptions and linearly fuses a semantic view with a one-hot identity view; GRACE extends it to mixed data via LLM-based numerical discretization, dual-view neighborhood consistency, and spectral clustering.
Four LLM backends, namely GPT-5.1, Claude Opus 4.5, DeepSeek V3.2, and Gemini 3 Pro, are considered and evaluated in \hyperref[sec:robustness]{\S\ref{sec:robustness}}.
For sentence encoding, $f_\theta$ is instantiated as all-mpnet-base-v2.

\textbf{Twenty-two public datasets} from the UCI Machine Learning Repository\footnote{\url{https://archive.ics.uci.edu/}}~\cite{dua2017uci} are used for evaluation. Eleven contain purely categorical attributes, and the remaining eleven mix categorical and numerical attributes. \hyperref[tab:datasets]{Table~\ref{tab:datasets}} reports the statistics of each dataset.

\textbf{Three validity indices} are adopted to evaluate clustering quality, namely the Adjusted Rand Index (ARI)~\cite{hubert1985comparing}, Normalized Mutual Information (NMI)~\cite{strehl2002cluster}, and Clustering Accuracy (ACC)~\cite{xu2003document}. Higher values indicate better performance for all three metrics. All experiments are repeated 10 times with different random seeds.

\subsection{Clustering Performance Evaluation}
\label{sec:performance}
\hyperref[tab:pure_cat_results]{Table~\ref{tab:pure_cat_results}} and \hyperref[tab:mixed_results]{Table~\ref{tab:mixed_results}} report the categorical and mixed benchmarks, where GRACE attains the best $\overline{AR}$ on every metric, with a categorical ARI rank of $1.15$, i.e., a first-place finish on nearly every dataset. The size of the lead points to a qualitative gain, since the LLM-grounded semantics recover relations that co-occurrence statistics cannot supply, e.g., the ordering of clinical severities or the kinship among biological traits. No competing method stays on top of both benchmarks, because the strongest categorical specialists do not carry over to mixed inputs, whereas the mixed-capable GUDMM-S slips to mid-pack once symbolic attributes dominate, leaving GRACE and GRACE-A as the only methods placed first and second throughout. GRACE-A trails GRACE by a fraction of a rank and still outranks every competitor, with its rare sizeable drops confined to small datasets such as ZO and SH, where the anchor set draws on too few points to preserve the fine structure.

A potential concern is that prompting the LLM exposes the dataset identity and lets the model retrieve the published labels online, shaping the value descriptions toward them. To rule out such leakage, we relabel two datasets (shown in dark red) with a composite label that a LLM builds from attribute names and value distributions, with no access to the original. The two attributes that define it are held out from every method, and an NMI audit keeps only labels near-independent of the published one (NMI $\le 0.13$ here), which places the new label beyond the reach of memory or a web search. GRACE still ranks first in ARI on both, which a method exploiting retrieved labels could not achieve, i.e., the advantage does not stem from label leakage. Every method scores low on these intrinsically hard targets, and the evidence therefore rests on the relative order, where GRACE leads even the baselines fed the ground-truth label as a feature.

Against methods built specifically for LLM-driven clustering (\hyperref[tab:llm_clustering_ari]{Table~\ref{tab:llm_clustering_ari}}), GRACE and GRACE-A again rank first and second, and the strongest prior design BREVE trails at $2.55$. Since those baselines also query an LLM yet land well behind, the gain traces to the $4P$ description and the cross-view verification, not to LLM access itself.

\begin{table}[!t]
\centering
\caption{ARI performance of representations generated by One-Hot (OH) encoding, HARR, and GRACE under $k$-means, hierarchical clustering with average linkage (HC-avg), and spectral clustering algorithms. ``\up''~indicates that GRACE outperforms both baselines.}
\label{tab:representation_effectiveness}
\resizebox{\columnwidth}{!}{%
\renewcommand{\arraystretch}{0.8}
\begin{tabular}{c|ccl|ccl|ccl}
\toprule
\multirow{2}{*}{\textbf{Dataset}} 
& \multicolumn{3}{c|}{\textbf{$k$-Means}} & \multicolumn{3}{c|}{\textbf{HC-avg}} & \multicolumn{3}{c}{\textbf{Spectral Clustering}} \\
\cmidrule(lr){2-4}\cmidrule(lr){5-7}\cmidrule(lr){8-10}
& OH & HARR & GRACE & OH & HARR & GRACE & OH & HARR & GRACE \\
\midrule
LE & 0.085 & 0.227 & 0.166 & -0.043 & -0.043 & 0.074\up & 0.182 & 0.166 & 0.002 \\
CS & 0.078 & 0.094 & 0.078 & -0.006 & 0.063 & 0.112\up & 0.064 & 0.094 & 0.078 \\
ZO & 0.702 & 0.501 & 0.674 & 0.889 & 0.960 & 0.799 & 0.523 & 0.657 & 0.650 \\
AA & 0.566 & 0.017 & 0.090 & 0.325 & 0.299 & 0.561\up & 0.592 & 0.436 & 0.595\up \\
LY & 0.197 & 0.194 & 0.199\up & 0.111 & 0.144 & 0.154\up & 0.234 & 0.228 & 0.238\up \\
TA & 0.037 & -0.008 & 0.010 & 0.037 & 0.017 & 0.063\up & -0.002 & 0.011 & 0.065\up \\
AM & 0.053 & 0.099 & 0.092 & 0.059 & -0.003 & -0.001 & 0.062 & 0.066 & 0.067\up \\
SO & 0.414 & 0.423 & 0.476\up & 0.343 & 0.303 & 0.522\up & 0.344 & 0.436 & 0.443\up \\
SH & -0.007 & 0.002 & 0.123\up & -0.103 & -0.107 & 0.031\up & -0.061 & -0.063 & 0.118\up \\
BC & -0.003 & -0.002 & -0.010 & 0.236 & 0.016 & 0.145 & 0.124 & 0.185 & 0.149 \\
HD & 0.383 & 0.412 & 0.402 & 0.270 & 0.284 & 0.198 & 0.386 & 0.420 & 0.412 \\
PT & 0.100 & 0.101 & 0.115\up & 0.067 & 0.058 & 0.129\up & 0.102 & 0.097 & 0.107\up \\
DE & 0.798 & 0.882 & 0.940\up & 0.514 & 0.616 & 0.836\up & 0.701 & 0.709 & 0.940\up \\
CK & 0.706 & 0.809 & 0.837\up & 0.070 & 0.702 & 0.880\up & 0.532 & 0.477 & 0.745\up \\
CV & 0.578 & 0.533 & 0.606\up & 0.004 & 0.004 & 0.599\up & 0.564 & 0.585 & 0.571 \\
MA & 0.425 & 0.442 & 0.426 & 0.000 & 0.314 & 0.319\up & 0.426 & 0.419 & 0.432\up \\
CA & 0.058 & 0.013 & 0.209\up & 0.013 & 0.013 & 0.147\up & 0.049 & 0.068 & 0.115\up \\
AV & -0.002 & 0.030 & -0.007 & 0.090 & 0.054 & 0.033 & 0.085 & 0.101 & 0.053 \\
OL & 0.277 & 0.286 & 0.264 & 0.001 & 0.001 & 0.270\up & 0.256 & 0.286 & 0.260 \\
SP & 0.612 & 0.642 & 0.683\up & -0.017 & 0.000 & 0.277\up & 0.311 & 0.436 & 0.405 \\
  MU & 0.207 & 0.166 & 0.215\up & 0.599 & 0.172 & 0.576 & 0.602 & 0.146 & 0.596 \\
  AD & 0.191 & 0.033 & 0.170 & -0.005 & -0.006 & 0.158\up & 0.159 & -0.006 & 0.140 \\
\midrule
 $\overline{AR}$ & 2.30 & 2.00 & \textbf{1.70} & 2.18 & 2.41 & \textbf{1.41} & 2.41 & 2.00 & \textbf{1.59} \\
\bottomrule
\end{tabular}%
}
\end{table}

\subsection{Ablation Study}
\label{sec:ablation}
To isolate each component, \hyperref[tab:ablation]{Table~\ref{tab:ablation}} compares the statistical view alone (I), that view plus semantic representation (SR, giving II), and the full model that adds dual-view neighborhood consistency (DVNC, giving Full). SR is the primary driver, as the $\overline{AR}$ falls from roughly $2.6$ to $1.9$ on all three metrics. DVNC then contributes a further drop to near $1.5$ by cross-checking the semantic affinities against statistical evidence. The benefit is clearest on SH and ZO. On SH, SR alone adds nothing, yet the full model lifts ARI from $0.068$ to $0.294$ once cross-view confirmation removes unreliable semantic neighbors. On ZO, SR even lowers the score, whereas the complete model climbs back above the baseline because the statistical view down-weights the misleading semantic affinities. DVNC therefore serves a dual purpose, reinforcing the correct semantic signals and repairing the harmful ones.

A further ablation in \hyperref[tab:mix_mechanism_ablation]{Table~\ref{tab:mix_mechanism_ablation}} turns to mixed data, where the baseline (I) reduces to k-prototype with Hamming for categorical and Euclidean for numerical attributes. Grounding the numerical distance alone (II) or the categorical distance alone (III) already improves on I, which shows the semantic metric suits either attribute type. Applying it to both (Full) wins by the widest margin, and the average rank drops steadily from $3.60$ at I to $1.20$ at Full across the four settings. The monotonic trend traces the gain to a single cause, i.e., the two attribute types come to share one semantic metric instead of the separate distances that k-prototype fuses by hand, which unifies them in a common metric space. Finally, three additional ablations on the four prompt perspectives, the dual-channel design, and the graph-construction and spectral-clustering steps appear in the \href{https://github.com/develop-yang/GRACE-GRACE-A}{\underline{supplementary material}}.

\subsection{General-Purpose Representation}
\label{sec:representation}
To test whether the GRACE representation generalizes beyond the spectral setting, \hyperref[tab:representation_effectiveness]{Table~\ref{tab:representation_effectiveness}} pairs three encodings, namely one-hot (OH), HARR, and GRACE, with $k$-means, hierarchical clustering under average linkage (HC-avg), and spectral clustering. The three representations differ in the information they carry, since OH treats every value as equidistant, HARR computes statistical distances from conditional distributions, and GRACE grounds the value relations in external semantic knowledge. GRACE holds the best $\overline{AR}$ under all three algorithms, whereas OH and HARR trade the runner-up position from one to the next, which shows that only the added world knowledge delivers a reliable gain. The margin widens most under HC-avg, because greedy agglomeration runs no iterative refinement and therefore exposes the quality of the input distances most directly. GRACE thus serves as a general-purpose representation that transfers across clustering paradigms, not a component tuned to the spectral case alone.

\subsection{Robustness to LLM Selection}
\label{sec:robustness}
Robustness to the LLM backend is examined by replacing the default GPT with Claude, DeepSeek, and Gemini, recording the ARI deviation on every benchmark (\hyperref[tab:llm_robustness_ari]{Table~\ref{tab:llm_robustness_ari}}). For GRACE, $|\Delta|$ stays within $0.02$ on the vast majority of datasets, and the Friedman test returns $p{=}0.41$, i.e., switching the model leaves performance statistically unchanged. The consistency follows from the design of the $4P$ prompt, because it queries established factual knowledge that LLMs trained on overlapping corpora encode in much the same way, and the three alternatives therefore produce nearly identical value descriptions. GRACE-A shows larger swings on a few datasets, e.g., a drop on CA under Gemini, because the anchor approximation magnifies small representation differences. Those swings favor no particular backend, however, and the test for the approximate variant returns an even higher $p{=}0.96$. Therefore, the effectiveness of GRACE rests on the description design and the cross-view verification, not on any single LLM.

\begin{table}[!t]
\centering
\caption{LLM robustness evaluation (ARI) on 22 benchmarks. The GPT column reports absolute performance, and each $\Delta$ column gives the deviation of an alternative backend from the GPT baseline, where Cla., DS, and Gem.\ denote Claude, DeepSeek, and Gemini. The Friedman $p$ row reports a Friedman test assessing whether the choice of backend significantly affects performance, with n.s.\ denoting a non-significant result ($p \ge 0.05$).}
\label{tab:llm_robustness_ari}
\renewcommand{\arraystretch}{0.8}
\setlength{\tabcolsep}{3pt}
\resizebox{\columnwidth}{!}{%
\begin{tabular}{c|crrr|crrr}
\toprule
\multirow{2}{*}{\textbf{Dataset}}
& \multicolumn{4}{c|}{\textbf{GRACE}} & \multicolumn{4}{c}{\textbf{GRACE-A}} \\
\cmidrule(lr){2-5}\cmidrule(lr){6-9}
& \textbf{GPT} & $\Delta$\textbf{Cla.} & $\Delta$\textbf{DS} & $\Delta$\textbf{Gem.}
& \textbf{GPT} & $\Delta$\textbf{Cla.} & $\Delta$\textbf{DS} & $\Delta$\textbf{Gem.} \\
\midrule
LE & 0.164 & 0.00 & 0.00 & 0.00 & 0.201 & 0.00 & 0.00 & 0.00 \\
CS & 0.094 & \textcolor{pos}{+0.02} & \textcolor{pos}{+0.05} & \textcolor{pos}{+0.05} & 0.094 & 0.00 & 0.00 & 0.00 \\
ZO & 0.827 & 0.00 & 0.00 & 0.00 & 0.617 & \textcolor{neg}{-0.01} & 0.00 & \textcolor{neg}{-0.01} \\
AA & 0.529 & 0.00 & \textcolor{pos}{+0.01} & 0.00 & 0.608 & \textcolor{pos}{+0.01} & \textcolor{pos}{+0.01} & \textcolor{pos}{+0.01} \\
LY & 0.194 & \textcolor{pos}{+0.02} & \textcolor{neg}{-0.01} & \textcolor{pos}{+0.01} & 0.148 & \textcolor{pos}{+0.02} & 0.00 & \textcolor{pos}{+0.01} \\
TA & 0.066 & 0.00 & 0.00 & \textcolor{neg}{-0.01} & 0.060 & \textcolor{neg}{-0.01} & 0.00 & \textcolor{neg}{-0.01} \\
AM & 0.093 & \textcolor{pos}{+0.02} & \textcolor{neg}{-0.01} & \textcolor{neg}{-0.01} & 0.086 & \textcolor{neg}{-0.01} & 0.00 & \textcolor{neg}{-0.01} \\
SO & 0.456 & 0.00 & 0.00 & 0.00 & 0.444 & \textcolor{pos}{+0.01} & \textcolor{neg}{-0.01} & \textcolor{pos}{+0.02} \\
SH & 0.294 & 0.00 & 0.00 & 0.00 & 0.134 & 0.00 & 0.00 & 0.00 \\
BC & 0.145 & \textcolor{pos}{+0.01} & \textcolor{pos}{+0.01} & 0.00 & 0.155 & \textcolor{pos}{+0.01} & \textcolor{pos}{+0.01} & 0.00 \\
HD & 0.409 & \textcolor{neg}{-0.01} & \textcolor{neg}{-0.02} & 0.00 & 0.417 & \textcolor{neg}{-0.02} & \textcolor{neg}{-0.04} & 0.00 \\
PT & 0.114 & 0.00 & 0.00 & 0.00 & 0.116 & \textcolor{neg}{-0.01} & 0.00 & 0.00 \\
DE & 0.895 & \textcolor{pos}{+0.03} & 0.00 & \textcolor{pos}{+0.01} & 0.887 & \textcolor{pos}{+0.02} & 0.00 & \textcolor{pos}{+0.01} \\
CK & 0.896 & 0.00 & \textcolor{pos}{+0.01} & \textcolor{neg}{-0.02} & 0.814 & 0.00 & 0.00 & 0.00 \\
CV & 0.628 & 0.00 & 0.00 & 0.00 & 0.580 & 0.00 & 0.00 & 0.00 \\
MA & 0.402 & 0.00 & \textcolor{neg}{-0.02} & \textcolor{neg}{-0.01} & 0.403 & \textcolor{neg}{-0.01} & \textcolor{neg}{-0.03} & 0.00 \\
CA & 0.131 & 0.00 & \textcolor{pos}{+0.04} & \textcolor{pos}{+0.01} & 0.092 & \textcolor{pos}{+0.01} & \textcolor{pos}{+0.04} & \textcolor{neg}{-0.08} \\
AV & 0.053 & 0.00 & 0.00 & 0.00 & -0.002 & 0.00 & 0.00 & 0.00 \\
OL & 0.255 & 0.00 & \textcolor{pos}{+0.02} & \textcolor{neg}{-0.01} & 0.250 & \textcolor{pos}{+0.03} & \textcolor{pos}{+0.03} & \textcolor{pos}{+0.02} \\
SP & 0.396 & 0.00 & \textcolor{neg}{-0.03} & \textcolor{neg}{-0.01} & 0.366 & 0.00 & \textcolor{neg}{-0.02} & \textcolor{neg}{-0.01} \\
\midrule
\textbf{Friedman $p$}
 & \multicolumn{4}{c|}{$0.41$ (n.s.)}
 & \multicolumn{4}{c}{$0.96$ (n.s.)} \\
\bottomrule
\end{tabular}%
}
\end{table}

\subsection{Scalability Evaluation}
\label{sec:scalability}
Scalability is evaluated on synthetic data generated from the Dermatology dataset by sampling from its per-class conditional distributions, with the sample-size experiment varying $n$ from 10{,}000 to 100{,}000 in steps of 10{,}000 at $s{=}20$ attributes and $k{=}5$ clusters (\hyperref[fig:scalability]{Figure~\ref{fig:scalability}}). As an exact spectral method, GRACE inherits an $O(n^2)$ memory cost and gives out beyond 20{,}000, later than dense baselines such as AMPHM that fail earlier. GRACE-A removes the ceiling and reaches 100{,}000 with near-linear scaling, a scale no competing method attains without sacrificing clustering quality. In fact, the approximate variant attains the highest ARI throughout and edges slightly above GRACE at the two overlapping scales, because the anchor sampling acts as an implicit regularizer on the well-separated synthetic clusters. The supporting analysis is deferred to the \href{https://github.com/develop-yang/GRACE-GRACE-A}{\underline{supplementary material}}.
The cardinality and cluster-count experiments fix $n{=}2{,}000$ and $s{=}20$, sweeping $V$ from 100 to 1{,}000 (at $k{=}5$) and $k$ from 20 to 200, respectively. Several counterparts show steep time growth or memory failures as either quantity rises, and the full model itself climbs steadily with the value count, whereas GRACE-A stays nearly flat across both sweeps. Taken together, the approximate variant offers a practical drop-in replacement once the quadratic cost of GRACE turns prohibitive.
\begin{figure}[t]
    \centering
    \includegraphics[width=\columnwidth]{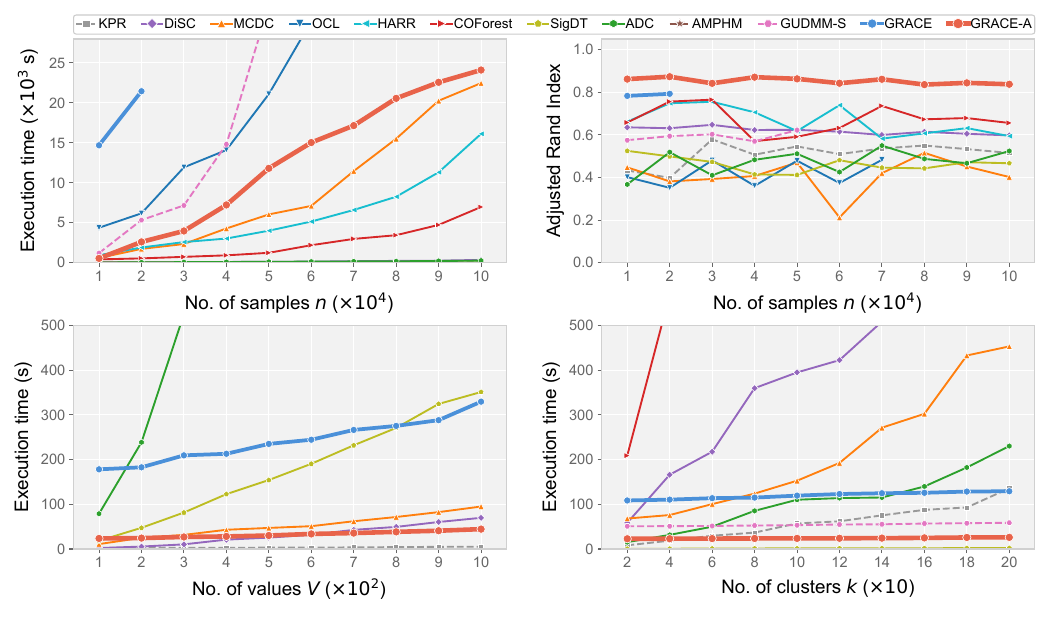}
    \Description{}
    \caption{Scalability evaluation on synthetic data. The sample-size experiment varies $n$ from 10{,}000 to 100{,}000, reporting execution time and ARI. The cardinality experiment varies $V$ from 100 to 1{,}000. The cluster-count experiment varies $k$ from 20 to 200. Missing points indicate out-of-memory failures.}
    \label{fig:scalability}
\end{figure}

\subsection{Case Study on CK Dataset}
\label{sec:case}

\begin{figure}[t]
    \centering
    \includegraphics[width=\columnwidth]{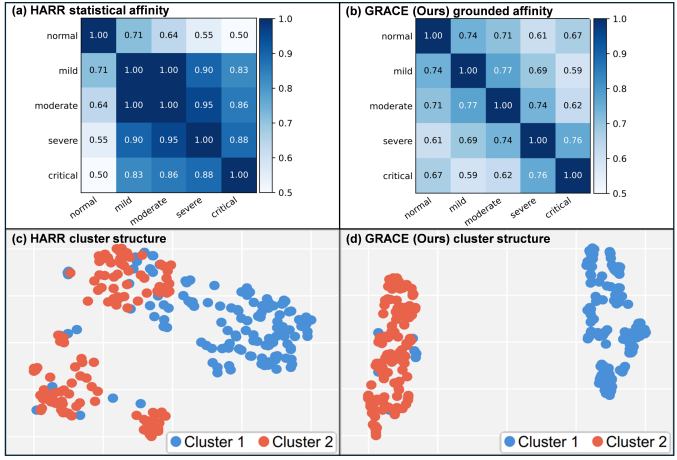}
    \Description{Case study.}
    \caption{Case study on the serum creatinine attribute of the CK dataset, where GRACE is our method and HARR is the strongest statistics-only baseline. The top row reports the value-level affinity among the five severity levels, (a) under HARR and (b) under GRACE. The bottom row shows the UMAP of the sample representations, (c) under HARR and (d) under GRACE, colored by cluster assignment.}
    \label{fig:case}
\end{figure}

Among the categorical attributes considered so far, ordinal ones additionally carry an inherent order over their values, which makes them a clean test of whether a method's affinity reflects genuine structure (\hyperref[fig:case]{Figure~\ref{fig:case}}). The serum creatinine attribute of the CK (Chronic Kidney Disease) dataset offers one such case, with five levels running from normal to critical. Because GRACE grounds every value in external world knowledge before consulting dataset statistics, it encodes each level into a description representation, and the cosine similarity between two such representations defines the affinity reported in Panel~(b). The grounded affinity therefore tracks clinical meaning, falling off as severity distance grows, e.g., severe remains near critical at 0.76 yet drops to 0.61 against normal, and the five levels settle into a smooth gradient along the scale. Reading the recovered structure then uncovers an ordinal arrangement that GRACE was never instructed to produce, i.e., the gradient emerges as a by-product of reasoning over meaning instead of counting co-occurrences. HARR, by comparison, derives its affinity from dataset statistics alone, and Panel~(a) reveals the resulting distortion. It assigns mild and moderate an identical score of 1.00 and pushes normal away from every abnormal level, thereby flattening the fine-grained progression that separates one clinical stage from the next. The contrast at the value level in turn carries over to the sample level without further intervention. The faithful gradient of GRACE divides the patients into two compact clusters in Panel~(d), whereas the collapsed affinity of HARR leaves the same patients interleaved in Panel~(c). Overall, semantic grounding enables GRACE to recover the genuine structure of an ordinal attribute, and the correct linear order follows as a natural product of that grounding and not as a target built into the model.

\section{Concluding Remarks}
In modern data engineering, clustering mixed tabular data serves as a fundamental technique for the label-free organization, cohort discovery, and exploratory analysis of heterogeneous records. To advance this critical capability, this paper presents GRACE, a scalable framework that practically injects LLM-derived world knowledge into the clustering pipeline. To circumvent the prohibitive token and temporal overheads inherent to iterative LLM queries, GRACE shifts semantic acquisition to the attribute-value level. This one-shot grounding extracts reusable semantic anchors, breaking the scalability bottleneck of LLM-assisted metric learning. Furthermore, to guarantee the faithfulness of the external knowledge and prevent semantic hallucinations from distorting actual data groupings, GRACE strictly cross-validates these conceptual affinities against dataset-internal statistical evidence. This dual-view calibration yields a robust metric space, ensuring the resulting clusters achieve superior accuracy alongside conceptual interpretability. Building upon this architecture, we further introduce GRACE-A, an approximation-based variant that reduces the computational complexity to linear time. GRACE-A delivers an ultra-efficient yet precise solution, ensuring that LLM-enhanced metric learning remains computationally feasible even for massive-scale datasets.

Despite the above-mentioned merits, the current one-shot LLM grounding mechanism inherently assumes an offline batch setting, requiring the complete attribute value space to be fully observable before metric construction. Given the increasingly frequent updates in modern databases, extending our framework to efficiently accommodate emerging distributional shifts of mixed data on the fly presents an important avenue for future research. 

\bibliographystyle{IEEEtran}
\bibliography{References}

\end{document}